\documentclass[letterpaper]{article} % DO NOT CHANGE THIS
\usepackage{aaai25}  % DO NOT CHANGE THIS
\usepackage{times}  % DO NOT CHANGE THIS
\usepackage{helvet}  % DO NOT CHANGE THIS
\usepackage{courier}  % DO NOT CHANGE THIS
\usepackage[hyphens]{url}  % DO NOT CHANGE THIS
\usepackage{graphicx} % DO NOT CHANGE THIS
\usepackage{natbib}  % DO NOT CHANGE THIS AND DO NOT ADD ANY OPTIONS TO IT
\usepackage{caption} % DO NOT CHANGE THIS AND DO NOT ADD ANY OPTIONS TO IT
\usepackage{algorithm}
\usepackage{multirow}
\usepackage{algpseudocode}[noend]
\usepackage{enumitem}
\usepackage{graphicx,subfigure}
\usepackage{amsthm,amsmath,amssymb}
\usepackage{comment}
\usepackage{colortbl}
\usepackage{xcolor}
\usepackage{makecell}
\usepackage{newfloat}
\usepackage{listings}
\DeclareCaptionStyle{ruled}{labelfont=normalfont,labelsep=colon,strut=off} % DO NOT CHANGE THIS
\floatstyle{ruled}
\newfloat{listing}{tb}{lst}{}
\floatname{listing}{Listing}
\nocopyright %-- Your paper will not be published if you use this command
\newtheorem{theorem}{Theorem}
\newtheorem{property}{Property}
\newtheorem{definition}{Definition}

\title{DTW+S: Shape-based Comparison of Time-series with Ordered Local Trends}

\author {
    Ajitesh Srivastava
}
\affiliations {
    University of Southern California\\
    ajiteshs@usc.edu
}

\begin{document}

\maketitle

\begin{abstract}
Measuring distance or similarity between time-series data is a fundamental aspect of many applications including classification and clustering. Existing measures may fail to capture similarities due to local trends (shapes) and may even produce misleading results. Our goal is to develop a measure that looks for similar trends occurring around similar times and is easily interpretable for researchers in applied domains. This is particularly useful for applications where time-series have a sequence of meaningful local trends that are ordered, such as in epidemics (a surge to an increase to a peak to a decrease). We propose a novel measure, DTW+S, which creates an interpretable ``closeness-preserving'' matrix representation of the time-series, where each column represents local trends, and then it applies Dynamic Time Warping to compute distances between these matrices. We present a theoretical analysis that supports the choice of this representation. We demonstrate the utility of DTW+S in ensemble building and clustering of epidemic curves. We also demonstrate that our approach results in better classification compared to Dynamic Time Warping for a class of datasets, particularly when local trends rather than scale play a decisive role.
\end{abstract}

\section{Introduction}

The distance between two time-series is a fundamental measure used in many applications, including classification, clustering, and evaluation. In classification and clustering, we want two ``similar'' time-series to have a low distance between them so that they can be grouped together or classified with the same label. In evaluation, the setting could be that we are generating projections (long-term forecasts) of time-series, and retrospectively, we wish to measure how close we are to the ground truth. 

While many measures exist for these purposes, including Euclidean distance, correlation, and dynamic time-warping (DTW)~\cite{muller2007dynamic}, the choice of the similarity measure can depend on the domain. Further, existing similarity measures may fail to capture the desired properties of the task at hand, for instance, when we wish to capture the similarity in trends over time. For example, consider the scenario presented in Figure~\ref{fig:motivation_scale}. Two models perform a projection to estimate the time-series given by the ground truth. Model 1 produces a pattern that is similar to the ground truth, while Model 2 produces a flat line. If we use mean absolute error to assess which model performed better, Model 2 (flat line) will receive a better score. Although Model 1 produces identical trends and correctly predicts the peak timing, it loses to a Model 2 which conveys no information. Now, consider the scenario presented in Figure~\ref{fig:motivation_shift}. Model 1 predicts the exact pattern but it slightly shifted in time. Again, Model 1 -- a flat line, produces a lower error.  Finally, in Figure~\ref{fig:motivation_bimodal}, Model 1 predict the overall pattern well, it only misjudges the height of the peaks. Yet, Model 2, a straight line, is considered closer to the ground truth. Some form of a range normalization could have addressed the issue in Figure~\ref{fig:motivation_scale}, and Dynamic Time Warping (DTW)~\cite{muller2007dynamic}, which allows stretching the time dimension to best match two time-series, can address the issue raised in Figure~\ref{fig:motivation_shift}. However, DTW and/or any normalization of scale cannot address the issue presented in Figure~\ref{fig:motivation_bimodal}.

\begin{figure*}
    \centering
    \subfigure[Similar trend different scale]{
    \includegraphics[width=0.31\textwidth,trim=0.7cm 0cm 0.5cm 0cm,clip]{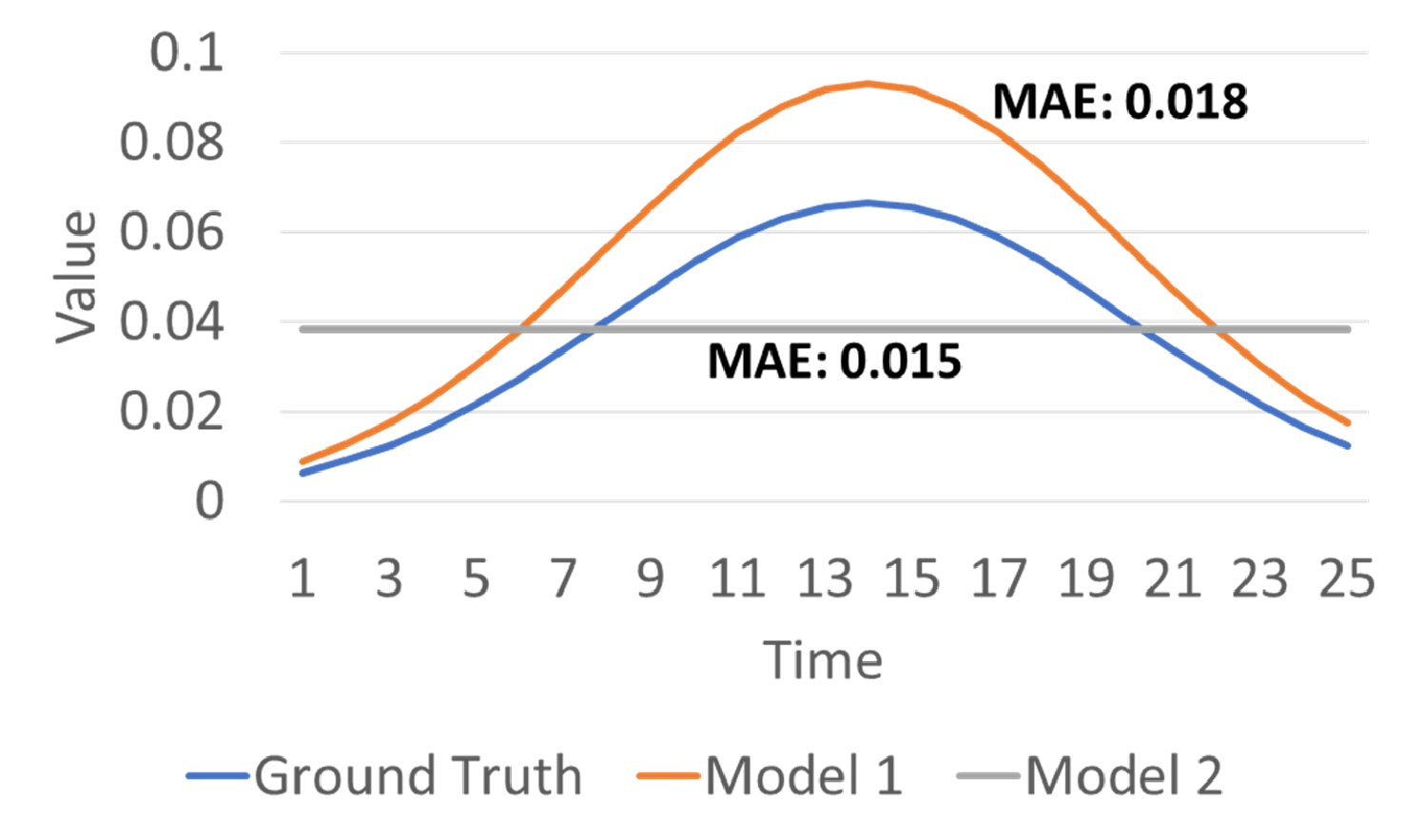}
    %\caption{Similar trend different scale}
    \label{fig:motivation_scale}
    }
    \subfigure[Similar scale, shifted in time]{
    \includegraphics[width=0.31\textwidth,trim=0.7cm 0cm 0.5cm 0cm,clip]{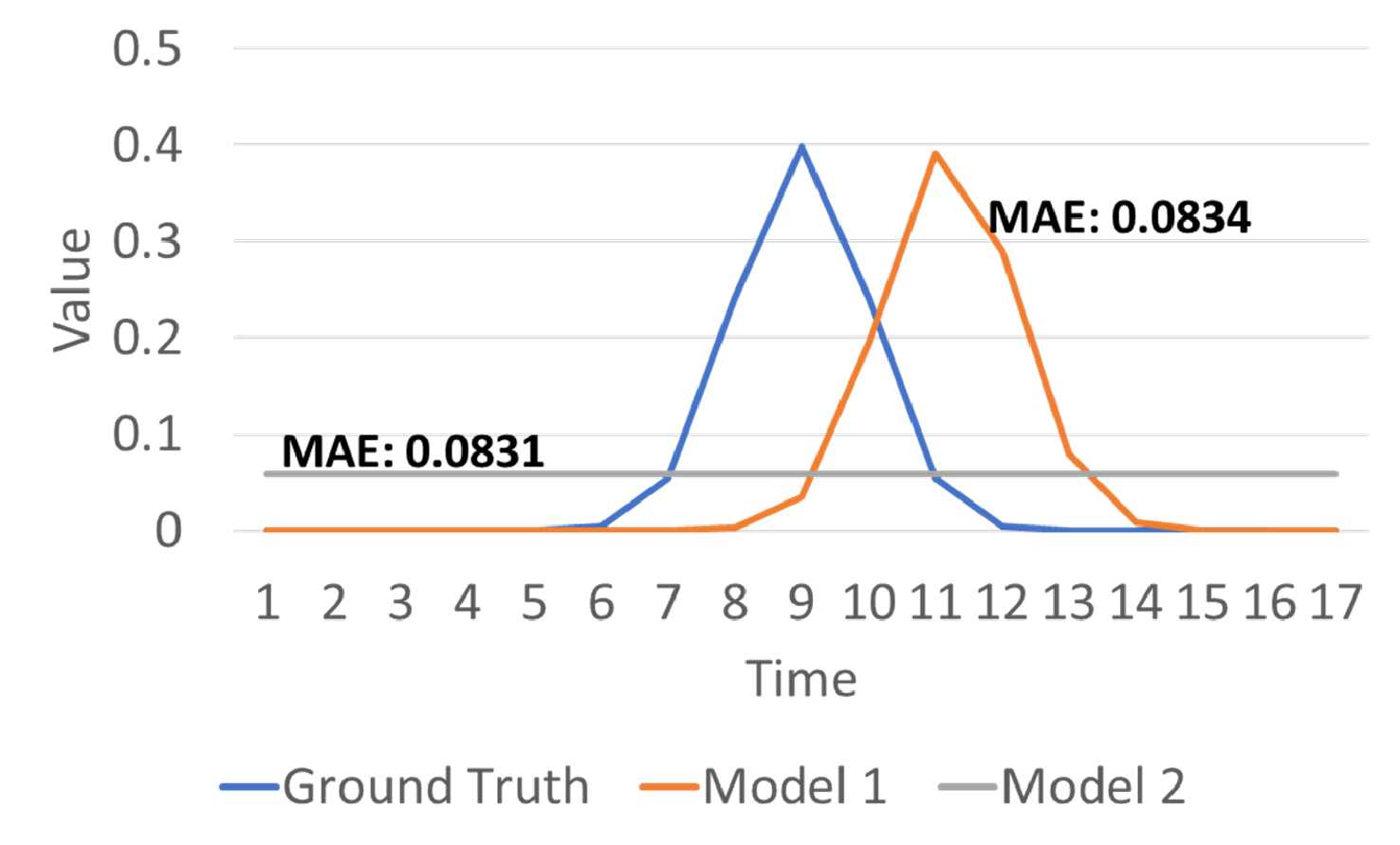}
    %\caption{Similar scale, shifted in time}
    \label{fig:motivation_shift}
    }
    \subfigure[Similar trends with varying scales]{
    \includegraphics[width=0.31\textwidth,trim=0.7cm 0cm 0.5cm 0cm,clip]{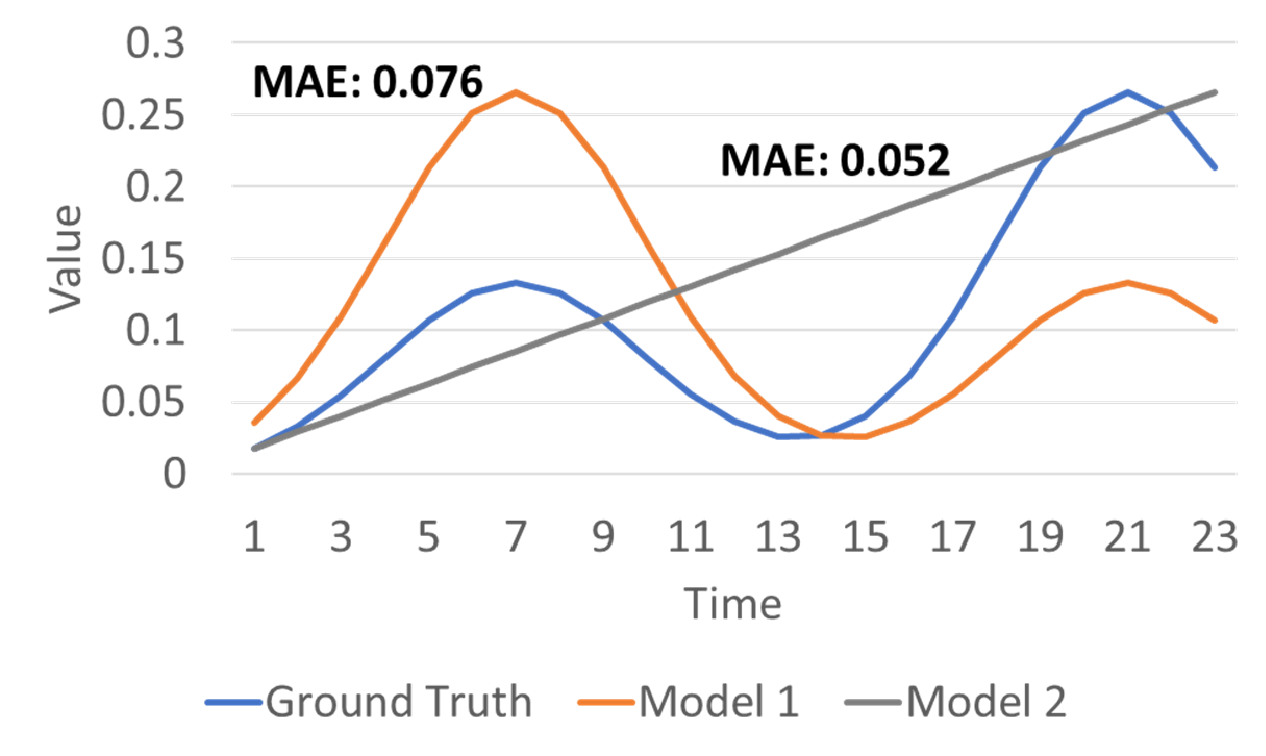}
    %\caption{Similar trends with varying scales}
    \label{fig:motivation_bimodal}
   } 
    \caption{Simple measures like Mean Absolute Error can be deceiving. In all these examples, Model 1 seems to be closer to the Ground truth, but receives a higher distance compared to a straight line.}
    \label{fig:enter-label}
\end{figure*}

Our goal is to develop a measure of distance such that two time-series are considered similar if and only if they have a similar sequence of trends and similar trends occur around similar times. This is particularly useful in public health where the time-series may represent meaningful local trends that are ordered, e.g., a surge, followed by an increase, then a peak, and finally a decrease.
We define a trend as the local shape of a time-series. Further, we wish this measure to be easily \textbf{interpretable} so that it can be adopted by researchers in applied domains, such as epidemiology.  To achieve this, we propose a novel distance measure DTW+S that (i) produces a matrix representation where each column encodes local trends, and (ii) uses Dynamic Time Warping on these matrices to compute distances between a pair of time-series. With this measure, we are able to perform better clustering and ensembling~\cite{SMH,fluSMH,ESMH} of epidemic curves where local interpretable trends are of interest.
While we do not intend to develop a new time-series classification algorithm, we believe a good distance measure should improve a simple distance-based classifier such as $k$-nearest neighbors~\cite{petitjean2014dynamic} on a class of tasks. Therefore, we also demonstrate the success of 1-nearest neighbor using our measure on classification.
% Specifically, our contributions are as follows:
% \begin{itemize}
%     \item We prove necessary and sufficient conditions for the shapelet space representation~\cite{srivastava2022shape} to be neighborhood-preserving -- two local trends are similar if and only if they are mapped to nearby points (Section~\ref{sec:choosing_shapelets}).
%     \item We propose a novel distance measure for time-series DTW+S that leverages this shapelet representation along with Dynamic Time Warping to find if similar trends occur around similar times (Section~\ref{sec:DTW+S}).
%     \item We develop an ensembling technique using DTW+S that can simultaneously summarize time-series in scale and time (Section~\ref{sec:ens_gen}), and  we demonstrate its utility on epidemic curves (Section~\ref{sec:exp_ens}).
%     \item We demonstrate that DTW+S results in more sensible clustering of epidemic curves (Section~\ref{sec:exp_cluster}).
%     \item We demonstrate that the DTW+S outperforms Dynamic Time Warping in classifying time-series on a subset of datasets, particularly, those where local trends play a key role in classification (Section~\ref{sec:exp_classify}).
% \end{itemize}

\subsubsection{Contributions}
(1) We \textit{prove necessary and sufficient conditions} for the shapelet space representation~\cite{srivastava2022shape} to be neighborhood-preserving -- two local trends are similar if and only if they are mapped to nearby points (Section~\ref{sec:choosing_shapelets}).
(2) We propose a novel distance measure for time-series DTW+S that leverages this representation along with DTW to find if similar trends occur around similar times (Section~\ref{sec:DTW+S}).
(3) We develop an ensembling technique that combined DTW+S with barycenter averaging~\cite{petitjean2011global} that can simultaneously summarize time-series in scale and time (Section~\ref{sec:ens_gen}), and  we demonstrate its utility on epidemic curves (Section~\ref{sec:exp_ens}).
(4) We demonstrate that DTW+S results in more sensible clustering of epidemic curves (Section~\ref{sec:exp_cluster}). Also, DTW+S outperforms Dynamic Time Warping in classifying time-series on a subset of datasets, particularly, those where local trends play a key role in classification (Section~\ref{sec:exp_classify}).
\section{Related Work}

\subsection{Background}

\subsubsection{Shapelet Space Representation}
In~\cite{srivastava2022shape}, the idea of the shapelet space representation is introduced to compare short-term forecasts of epidemics. The motivation is to compare the shape of the forecasts rather than exact numerical values. Further, they wish to make the representation interpretable. Each dimension represents the similarity of the vector with one of the chosen shapes of interest, such as an increase $(1, 2, 3, 4)$ and peak $(1, 2, 2, 1)$. These shapes of interest are termed Shapelets. 
\begin{definition}[Shapelet]
 A shapelet $\mathbf{s} = [s_1, \dots, s_w] \in \mathbb{R}^w$ is a vector that represents a shape of interest.
\end{definition}
\begin{definition}[Shapelet-space Representation]
Given $d$ shapelets $\{\mathbf{s_1}, \dots \mathbf{s_d}\}$, a Shapelet-space Representation is a vector $\mathbf{x}$ as a $d$-dimensional point $P_x = (p_1, p_2, \dots, p_d)$ denoting the shape of  $\mathbf{x} \in \mathbb{R}^w$ and the co-ordinate $p_i = sim(\mathbf{x}, p_i)$ for some measure of similarity. The function $f: \mathbb{R}^w \rightarrow \mathbb{R}^d$ is the Shapelet-space Transformation.
\end{definition}
The similarity function is to be chosen in such a way that two shapes are considered similar if and only if one shape can be approximated by translation and scaling of the other. However, this may cause an issue -- when the shape is close to a ``flat'', small noise can cause it to become similar to other shapes. It is argued that there is an inherent concept of flatness in the domain of interest. For instance, in influenza when the number of hospitalizations is stable at a very low value, that shape is to be considered flat and not to be considered similar to any other shape when hospitalizations are higher. Therefore, a desirable property is the following.
\begin{property}[Closeness Preservation]\label{prop:2}
Two vectors have similar representation, if and only if (i) none of the vectors are ``almost flat'' and one can be approximately obtained by scaling and translating the other, or (ii) both vectors are ``almost flat''.  
\end{property}
They propose an approach that first identifies how similar a shape is to what we could consider ``flat'', and then updates the similarities of the shape with respect to other shapelets. 
For some constants $m_0, \beta \geq 0$, define ``flatness'' as 
% \begin{displaymath}
%        \phi =
%         \left\{\begin{array}{@{}cl}
%                 1, & \text{if } m \leq m_0,\\
%                 \exp(-\beta(m-m_0)),   & \text{if } m \geq m_0.
%         \end{array}\right.
% \end{displaymath}
$\phi = \exp(-\beta(m-m_0)),$ if $m > m_0$, otherwise $\phi=1$.
Here $m$ is the average absolute slope of the vector $\mathbf{x}$ whose shapelet-space representation is desired, i.e., if $\mathbf{x} = (x_1, x_2, x_3, x_4)$, then $m = (|x_2 - x_1| + |x_3-x_2| + |x_4 - x_3|)/3$. The constant $m_0$ enforces that a vector with a very small average absolute slope is considered flat and receives a $0$ similarity in all other dimensions.
The constant $\beta$ represents how quickly above the threshold $m_0$, the ``flatness'' should reduce. Now, the co-ordinates of shapelet-space representation are defined as
\begin{displaymath}
       sim(\mathbf{x}, \mathbf{s_i}) =
        \left\{\begin{array}{@{}cl}
                2\phi-1, & \text{if } \mathbf{s_i} \text { is ``flat''},\\
                (1-\phi)corr(\mathbf{x}, \mathbf{s_i}),   & \text{otherwise}.
        \end{array}\right.
\end{displaymath}
It is shown that this definition satisfies Closeness Preservation (Property~\ref{prop:2}) with $w$ or more shapelets including the ``flat'' shapelet. \textit{We prove that $w$ shapelets are not only sufficient but necessary to satisfy this property (Theorem~\ref{thm:necessary})}. We use Shapelet-space Representations of moving windows on the given time-series to capture local trends over time. 

\subsubsection{Dynamic Time Warping}
Dynamic Time Warping (DTW) is a distance measure between two time-series that allows warping (local stretching and compressing) of the time component so that the two time-series are optimally aligned. Given two time-series $\mathbf{a} = [a(1), a(2), \dots]$ and $\mathbf{b} = [b(1), b(2), \dots]$, the objective of DTW is to minimize $\sum_{i\leftrightarrow j} \mathcal{D}(a(i), b(j)$, for some distance measure $\mathcal{D}$, and where $i\leftrightarrow j$ represents aligning index $i$ of $\mathbf{a}$ with index $j$ of $\mathbf{b}$. The alignment is done under some constraints -- (1) if $a(i)$ and $b(j)$ are aligned then $a(i+1)$ cannot be aligned with $b(j')$ for some $j' < j$. (2) Every index is present in at least one alignment. (3) The first index of both time-series are aligned with each other. (4) The last index of both time-series are aligned with each other.
Further, a window constraint can be added~\cite{ratanamahatana2004making} suggesting that indices $i$ and $j$ can only be aligned if $|i-j| \leq w$, for some non-negative integer $w$. %DTW can be computed by dynamic programming in $O(T_1 T_2)$ space and time, where $T_1$ and $T_2$ are the lengths of the time-series. Various algorithms exist that improve the time and space complexity~\cite{keogh2000scaling}. Our measure leverages DTW while modifying the representation over which the distances are calculated.

\subsubsection{Time-series Ensemble}
In applications like epidemic projection, multiple trajectories are generated using different methods or different initializations. Then, an ensemble is created which is then communicated to the public and policy makers~\cite{SMH}. These ensembles are designed to capture the mean value at time $t$, i.e., for $n$ trajectories $\mathbf{a}_1, \dots, \mathbf{a}_n$, where $\mathbf{a}_i = [{a}(1), \dots, a(n)]$ the ensemble is $\bar{a}(t) = \sum_{i=1}^n a(t)/n$. As an unintended consequence, informative aspects of individual trajectories may be lost. As an example, consider Figure~\ref{fig:mean_ens}. Two models produce almost identical projections but they are shifted in time, and they have the same peak. The ensemble produces a trajectory that has a peak that is significantly lower and wider than individual models. In public health communication, this can cause a misjudgment of the severity of the epidemic. While the ensemble correctly summarizes the expected outcome at time $t$, the reader tends to infer other information such as peak timing and severity. Our similarity measure can be used to build ensembles that better preserve the properties of the individual models.

\begin{figure}[!ht]
    \centering
    \includegraphics[width=\columnwidth]{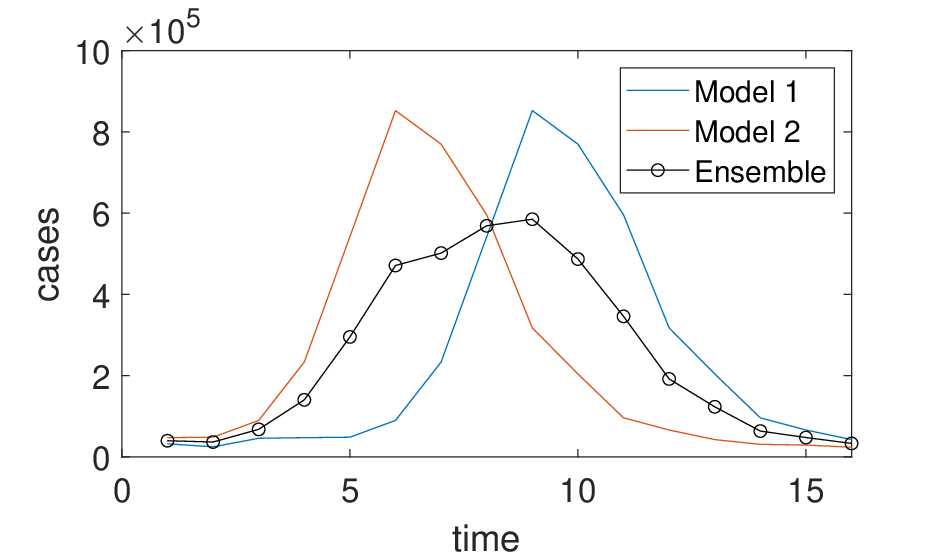}
    \caption{Failure of the mean ensemble in capturing the properties of individual time-series --  much lower peak.}
    \label{fig:mean_ens}
\end{figure}

\subsubsection{Shapelets} 
In time-series literature, ``shapelets'' have been used to refer to informative motifs that occur in time-series~\cite{ye2009time}. 
A feature vector for time-series can then be constructed by similarity of the best matching subsequence of the time-series to these motifs. The motifs are selected based on their representativeness of a class. In contrast, we use the term shapelet as \textit{a pre-determined shape of interest coming from the domain}. We prove that a specific class of shapelet sets and similarity measures is needed to develop a representation to satisfy the closeness-preserving property. We encode all local trends of the time-series into a matrix representation, which is compared using DTW.

\subsection{Related Similarity Measures}
While many similarity measures for time-series have been proposed in the literature~\cite{jeong2015support,lines2015time,dhamo2015comparing}, the closest work to our proposed measure DTW+S is ShapeDTW~\cite{zhao2018shapedtw}. It was developed as an alignment algorithm that captures point-wise local structures and preferentially aligns similarly-shaped structures. It does so by generating shape descriptors that include the raw sequence, piece-wise aggregation, Discrete Wavelet Transform~\cite{van2019discrete}, slopes, and derivatives. A key distinction from our approach is that we utilize a set of shapelets that are shapes of interest in the desired application, and hence our representation is directly interpretable. Further, we present theoretical results on the closeness-preserving characteristics of our approach. While shapeDTW is designed to be general-purpose by constructing a variety of shape descriptors, our approach is particularly designed for applications where the trend is more important than the scale. We later demonstrate that our approach still outperforms shapeDTW for classification on almost half of the 64 datasets considered (Section~\ref{sec:exp_classify}).

\section{Methodology}
\subsection{Definitions}
First, we define some terms used throughout. We start with the idea of a ``trend descriptor" that formally defines the idea of assigning arbitrary label (e.g., increase, decrease, peak, etc.) to a part of a time-series. This is intended to emulate a human using categories (implicit or explicit) to interpret a pattern in the time-series. This concept will help us define interpretability of a representation and similarity measure.
\begin{definition}[Trend Descriptor]
    A trend descriptor is a function $\mathcal{L}$ that maps any vector $\mathbf{x} \in \mathbb{R}^w$ to a label in set $L$ denoting a shape description of $\mathbf{x}$.
\end{definition}
For instance, a trend descriptor may map any given 4-element vector to one of slow increase, rapid increase, exponential increase, going to peak, going past a peak, rapid decrease, flat, and unknown. Some of these labels may be more similar to each other, e.g., slow and rapid increases are more similar to each other then rapid increase and decrease.

% \begin{theorem}
%     A map that satisfies closeness-preserving property can discriminate the labels of any arbitrary trend descriptor.
% \end{theorem}

\begin{definition}[Local Trend]
    For a time-series $(a_1, a_2, \dots, a_T)$, a window $w$ and some trend descriptor $\mathcal{L}$ we define the local trend at location $i$ as $\mathcal{L}(a_i, \dots, a_{i+w-1})$.
\end{definition}

\begin{definition}[Interpretable]
    We say a representation is interpretable if it is possible to identify the local trends based on the values in each dimension of the representation.
%    We say a distance measure is interpretable if it assigns low distance to two time-series if and only if they have the same local trend. 
\end{definition}

\begin{definition}[Ordered Local Trend]
    A class of time-series has ordered local trends if the order of the local trends appearing in the time-series conveys key information. In other words, the similarity between two time-series implies similarity between the sequences of local trends.
\end{definition}
In Figure~\ref{fig:motivation_bimodal} both orange and blue curves have the same sequence of local trends that can be described as a sequence of increases then peak, followed by a decrease, then a surge, and increase, another peak and decrease. While the gray line is only a sequence of increases. Such characterization is important in understanding and communicating long-term projections of epidemics as they represent specific events of interest~\cite{howerton2023informing,Flusight2023}. The interpretability and ordered local trends play a central role in our novel ensemble methods.
Before we use an alignment for ensembling, we develop the theory to be able to capture local trends (shapes) of the time-series for any implicit trend descriptor.
To do so, we extend the notion of Shapelet-space Representation to a time-series with a sliding window.

\begin{definition}[Shapelet-space Representation - Time-series] 
    Given a time-series $\mathbf{a} \in \mathbb{R}^{T_1}$, a window $w$, and a Shapelet-space Transformation $f:\mathbb{R}^w \rightarrow \mathbb{R}^d$, the Shapelet-space Representation of $\mathbf{a}$ is the matrix $\mathbf{A} \in \mathbb{R}^{d \times (T_1-w+1)}$ whose $i^{th}$ column is the Shapelet-space Representation of the vector $(a_i, \dots, a_{i+w-1})$.
\end{definition}
This matrix encodes how the time-series changes over time in an interpretable manner. Algorithm~\ref{alg:ts_shape} summarizes this approach.
For each time-series, $\mathbf{a}_i \in \mathbb{R}^{T_1}$, we first find its Shapelet-space Representation -- the matrix $\mathbf{A}_i \in \mathbb{R}^{d \times (T_1-w+1)}$. Each column (of size $d$) of these matrices is obtained by sliding a $w$-length window on the respective time-series and obtaining its $d$-dimensional Shapelet-space Representation (SSR). Figure~\ref{fig:SSR_example} shows the SSR obtained from a time-series. The SSR is built using four dimensions representing ``increase'', ``peak'', ``surge'', and ``flat''. A yellow color represents a high positive value and a blue represents a high negative value (e.g., a negative increase is a decrease). Also, note that ``surge'' and ``increase'' are similar shapes and hence seem to have a high correlation. The representation tells us that this time series has a sequence of surges/increases leading to a small peak (green in ``peak'' and ``flat'') around time step $5$, followed by stability, then increase, leading to a sharp peak (bright yellow around time-step $13$), followed by rapid decline (dark blue in ``inc'') and then stability (flatness).
%Finally, we use Dynamic Time Warping with a suitable window to align the time-series. The distance $\mathcal{D}$ is defined as the Euclidean distance between aligned columns of these matrices.  Extending the approach from two to more time-series alignments is presented in Section~\ref{sec:ens_gen}.
\begin{algorithm}
\caption{Extracting shapelet space representation of time-series using a moving window}\label{alg:ts_shape}
\begin{algorithmic}
\Procedure{TimeSeriesShape}{$\mathbf{a}, \mathbf{S}$}
 \State $w \gets $ length of each shapelet in $\mathbf{S}$
 \State $T \gets $ length of the time-series $\mathbf{a}$
 %\Comment{The output will be a $w\times T-w+1$ matrix}
 \For{$t=1$ to $T-w+1$}
    \State $\mathbf{A}[:, t]$ = ShapeTransform($\mathbf{a}[i:i+w-1], \mathbf{S}$)
    \Comment{Transform a window of length w into the shapelet space}
 \EndFor
 \State \Return $\mathbf{A}$
\EndProcedure
\end{algorithmic}
\end{algorithm}

%Next, we discuss how to choose a good Shapelet-space Representation.

\begin{figure}[!ht]
    \centering
    \includegraphics[width=0.85\columnwidth]{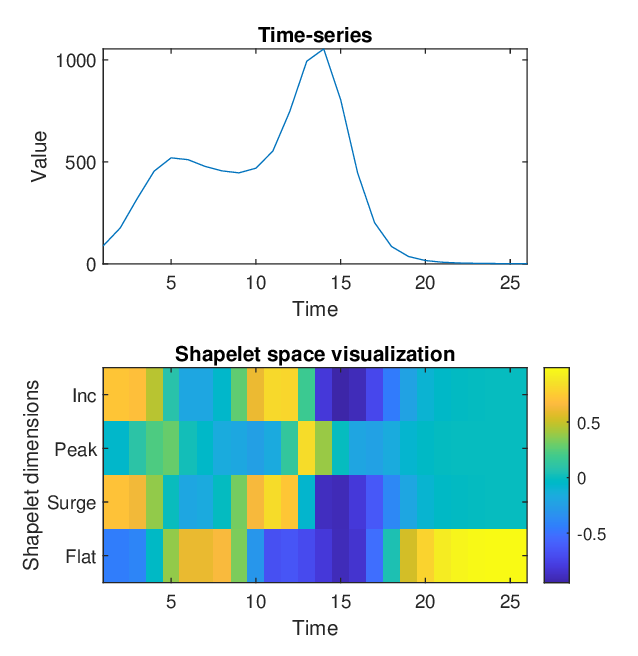}
    \caption{Shapelet-space Representation of a time-series.}
    \label{fig:SSR_example}
\end{figure}

\subsection{Choosing the Shapelet-Space}\label{sec:choosing_shapelets}
While Srivastava et. al. provide some indication of how to choose the set of shapelets, they only prove that for vectors with $w$ elements, $w$ shapelets are sufficient. While one may choose more number of shapelets, e.g., one for each trend descriptor in mind, having a large number ($d$) of shapelets also impacts the space and time complexities as both scale linearly with $d$. What is the minimum number of shapelets needed?
Here, we show that $w$ shapelets are not only sufficient but also necessary to satisfy the closeness-preserving property.
Let $f:\mathbb{R}^w \rightarrow \mathbb{R}^{d}$ be a shapelet transformation obtained by a set of linearly independent vectors $\mathbf{s_1}, \dots, \mathbf{s}_{d-1}$ and the flat vector $\mathbf{s}_0$.
Consider two vectors $\mathbf{x}$ and $\mathbf{y}$ of length $w$. Suppose $\mathbf{{x'}}$ and $\mathbf{{y'}}$ represent the corresponding normalized vectors obtained as:
    % \begin{equation}
    %     \mathbf{x'} = \frac{\mathbf{x} - \mu_\mathbf{x}}{\|\mathbf{x}\|} \mbox{ and } \mathbf{y'} = \frac{\mathbf{y} - \mu_\mathbf{y}}{\|\mathbf{y}\|}\,,
    % \end{equation}
    $
        \mathbf{x'} = \frac{\mathbf{x} - \mu_\mathbf{x}}{\|\mathbf{x}\|} \mbox{ and } \mathbf{y'} = \frac{\mathbf{y} - \mu_\mathbf{y}}{\|\mathbf{y}\|}\,,
    $
where $\mu_\mathbf{x}$ and $\mu_\mathbf{y}$ are the mean of elements in $\mathbf{x}$ and $\mathbf{y}$, respectively. The following can be shown (see Appendix for proofs).

\begin{theorem}\label{thm:sufficient}
Property~\ref{prop:2} is satisfied with any set of $w-1$ linearly independent shapelets and the ``flat'' shapelet, i.e., with this choice $\|f(\mathbf{x}) - f(\mathbf{y})\| \leq \epsilon$ iff (i) both $\mathbf{x}$ and $\mathbf{y}$ are ``almost'' flat, or (ii) $\|\mathbf{x'} - \mathbf{y'}\| \leq \delta$, for some small $\epsilon$ and $\delta$.
\end{theorem}

\begin{theorem}\label{thm:necessary}
    At least $w-1$ linearly independent shapelets are necessary with the ``flat'' shapelet to satisfy Property~\ref{prop:2}.
\end{theorem}

Next, we show that the proper choice of shapelets results in the desired expressive power of distinguishing any two labels of an arbitrary trend descriptor.

\begin{theorem}[Expressive power]
    A proper choice of Shapelet Space Transformation can discriminate any trend descriptor.
\end{theorem}

%\textbf{Expressive power: Discriminating any $\mathcal{L}$. } 
\begin{proof}
Due to the closeness preserving property, it follows that with $w$ shapelets as described by Theorems~\ref{thm:sufficient} and~\ref{thm:necessary}, we can distinguish between any two local trends taken from an arbitrary choice of scale-free trend descriptor $\mathcal{L}$. By scale-free, we mean a trend descriptor that does not distinguish based on scale, e.g., ``high increase" vs ``very high increase". On the other hand,
it should be noted that we do not completely ignore the scale information. The flat dimension can be rewritten as $sim(\mathbf{x}, \mbox{flat}) = 2\phi - 1 = 2\exp(-\beta m)$, choosing $m_0 = 0$. Given, the value in the flat dimension, we can uniquely find the average absolute slope $m$ and thus we are able to discriminate scale-based trend descriptors as well. Therefore, by choosing the $w$ shapelets appropriately, we can discriminate any two local trends taken from an arbitrary choice of $\mathcal{L}$. 
\end{proof}

\noindent\textbf{Remarks.} Any representation obtained by a function that maps the input into an element of a finite set cannot have the desired expressive power. Suppose the cardinality of the range of this function is $k$. Then, any trend descriptor of $k+1$ labels will have at least two trends that cannot be distinguished by this function.
For instance, consider a transformation from a vector to two categories based on the positive and negative slope (increase vs decrease). Then, a ``peak" cannot be distinguished from either. While the shapes of interest may be pre-determined, we wish to be able to distinguish them regardless of their precise definition. Further, note that the simple measure of absolute error also has the same expressive power, but it may also discriminate between the same local trends (both ``increasing'' but with slightly different slopes). Shapelet space transformation puts more emphasis on scale-free discrimination.

%\end{comment}

It follows that if we align two time-series that have the same ordered local trends but possibly at different times, Dynamic Time Warping on their SSR (DTW+S) will be able to align them, leading to a low distance (high similarity).

\subsection{Time-series Similarity: DTW+S}\label{sec:DTW+S}
Recall that our goal is to consider two time-series similar if they have similar shapes around the same time, but not necessarily \textit{at} the same time.
%The Shapelet-space Representation allows us to compare short-term trends. We also wish to allow comparison along nearby time points, i.e., $i^{th}$ column of $\mathbf{A}$ may not necessarily be compared with the $i^{th}$ column on $\mathbf{B}$.
To allow this, we use Dynamic Time Warping where the distance between aligning column $i$ with column $j$ is given by the square of the Euclidean distance between the $i^{th}$ column of $\mathbf{A}$ and $j^{th}$ column of $\mathbf{B}$, i.e.,  $\|\mathbf{A}[:, i] - \mathbf{B}[:, j]\|^2$. 
%Algorithm~\ref{alg:DTW+S} implements our approach. It calls the function \textit{TimeSeriesShape} to compute SSR of the input time-series.%, which is presented in Algorithm~\ref{alg:ts_shape}.
Figure~\ref{fig:explanation} provides a visualization of SSR matrices. Details of the interpretation are provided in Section~\ref{sec:explain}.
The choice of warping window $\tau$ depends on the application. For a classification task, it can be treated as a hyper-parameter and identified through validation on a held-out set. For epidemics, suppose, two models generate projections under the same assumptions, they may be predicting multiple peaks. Far-away peaks may refer to different events (e.g., two different variants causing two waves). However, peaks that occur 4-5 weeks apart across models may be referring to the same event. Therefore, $\tau = 5$ weeks would be a reasonable choice.

\subsection{Ensemble Generation}\label{sec:ens_gen}
The existing ensemble methods are designed to aggregate individual projections over time, thus measuring the scale (e.g., number of hospitalizations) at time $t$. They are not designed to aggregate when will an event (e.g., a peak) take place. However, a viewer tends to interpret both the scale and timing from the ensemble plot. We are interested in -- given $n$ time-series, $\mathbf{a}_i = [a_i(1), a_i(2), \dots, a_i(T)], i\in \{1, \dots, n\}$, find an ``ensemble'' time-series that captures the aggregate behavior in both time and scale. 
To address this, we assume that \textit{each trajectory tries to estimate a sequence of latent ``events''}. With this perspective, for some sequence of event $e_1, e_2, \dots$, a time-series captures the timing of $e_j$ and its scale. Therefore, the time-series can be interpreted as $\mathbf{a}_i = [(t_i(e_1), a_i(e_1)), (t_i(e_2), a_i(e_2)), \dots]$. For any given event $e_j$, the aggregate time-series can be obtained by averaging both the timing and the severity dimensions:
\begin{equation}\label{eqn:ensemble}
    \left(\bar{t}(e_j), \bar{a}(e_j)\right) = \left(\frac{\sum_i t_i(e_j)}{n}, \frac{\sum_i a_i(e_j)}{n}\right)
\end{equation}

However, we do not observe these ``events'' explicitly. We define an event to be reflected in the time-series by a local trend. When two time-series are aligned by DTW+S, each alignment corresponds to an event. Formally, in the shapelet space representation $\mathbf{A}$ and $\mathbf{B}$, if columns $\tau_1$ and $\tau_2$ are aligned, then the local trend at time $[\tau_1, \tau_1+w-1]$ in time-series $\mathbf{a}$ and that at time $[\tau_2, \tau_2+w-1]$ in time-series $\mathbf{b}$ are defined to be corresponding to the same ``event''. Suppose we use DTW+S to align $n$ projections. Then, each projection $i$ contributes one point $(t_i(j), a_i(t_j))$ for each alignment $j$. This results in a set of points, one for each alignment $j$, using Equation~\ref{eqn:ensemble}. Finally, if desired, we can interpolate these points to estimate the value of $\bar{a}(t)$ for $t \in \{1, 2, \dots, T\}$.
Note that this approach is based on the following assumptions. First, each time-series has a similar sequence of shapes but may differ in timing and severity/scale. Second, the interpolation assumes smoothness in the desired ensemble.

\begin{figure}
    \centering
    \includegraphics[width = 0.85\columnwidth]{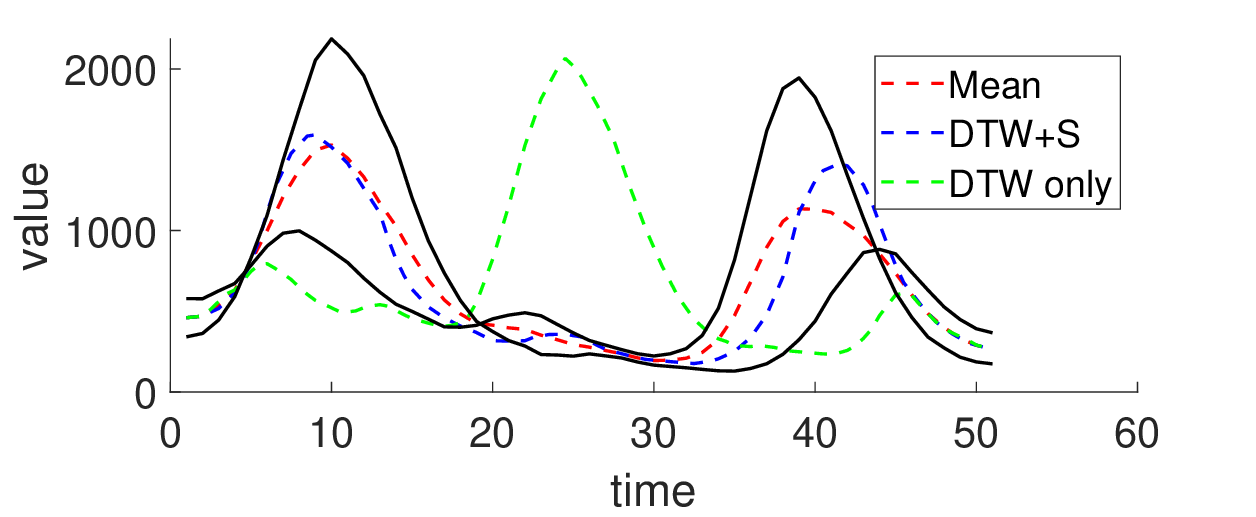}
    \caption{Applying mean, DTW, and DTW+S to develop ensemble of two time-series.}
    \label{fig:ens_DTW}
\end{figure}

Figure~\ref{fig:ens_DTW} demonstrates this approach for two time-series (in black solid line) that have two peaks each with different scales. We compare our approach against the mean ensemble and DTW (without SSR). Note that the mean ensemble results in a small second peak which is closer to the smaller peak among the input time-series. Further, its timing is biased towards the larger peak. The DTW ensemble results in one large peak by aggregating scale and timing of the first peak of one and the second peak of the other time-series. The DTW+S ensemble results in two peaks as expected, where each peak correctly averages the corresponding timing and scale of the input time-series.
While similar aggregation has been discussed in the literature~\cite{gupta1996nonlinear,petitjean2011global}, two key differences exist. First, our application is motivated by visual interpretation of ``events'', such as when the epidemic peaks. Second, our approach uses an interpretable shape-based measure instead of directly using the Euclidean distance on time-series. This allows us to define an ``event''.

\begin{algorithm}[!ht]
\caption{Barycenter Averaging with DTW on Shapes}\label{alg:DTW+SBA}
\begin{algorithmic}
\Procedure{DTW+SBA}{$\mathbf{a}, \mathbf{b}, \mathbf{S}, \tau$}

\For{$i=1$ to $n$}
        \State $\mathbf{A_i} \gets$ TimeSeriesShape($\mathbf{a}_i, , \mathbf{S}$)
\EndFor
\State $\mathbf{b} \gets$ medoid($\mathbf{a}_1, \mathbf{a}_2, \dots, \mathbf{a}_n$)
\While {not converge}
    \State $\mathbf{B} \gets $ TimeSeriesShape($\mathbf{b}, \mathbf{S}$)
    \For{$i=1$ to $n$}
        \State $\mathbf{t}_i\gets$ DTW\_align($\mathbf{A}_i,\mathbf{B}, \tau$)
        \Comment{$\mathbf{t}_i$ holds the indices of $\mathbf{A}_i$ aligned to $\mathbf{B}[:, 1], \mathbf{B}[:, 2], \dots,
\mathbf{B}[:, n]$.}
    \EndFor
    \For{$j = 1$ to $T$}
        \State $\mathbf{b}[j] \gets $ mean($\mathbf{a}[\mathbf{t}_1[j]], \mathbf{a}[\mathbf{t}_2[j]], \dots, \mathbf{a}[\mathbf{t}_n[j]]$) 
    \EndFor
\EndWhile

\State \Return $\mathbf{b}$
\EndProcedure
\end{algorithmic}
\end{algorithm}

Optimally aligning multiple time-series is NP-Hard with some approximation algorithms including DTW Barycenter Averaging~\cite{petitjean2011global}. In this approach, the initial `base' time-series is selected as the time-series among the set of trajectories that has the lowest distance from other trajectories. Then, in each iteration, we compute the pairwise alignment of all time-series with respect to the base time-series and the result becomes the new `base' time-series. The alignment is traditionally computed using DTW, which we replace with DTW+S. Algorithm~\ref{alg:DTW+SBA} summarizes this approach.
We demonstrate that the DTW+S based Barycenter Averaging better captures the properties of the trajectories (Section~\ref{sec:exp_ens}).

\subsection{Clustering and Classification}
\label{sec:class_and_clust}
We can use DTW+S to cluster time-series with any clustering or classification algorithm that allows customizable distance measures. As a demonstration, we use agglomeration clustering~\cite{day1984efficient} on the distance matrix where each entry $\mathcal{D}(\mathbf{a}_{i}, \mathbf{a}_{i'})$ is the DTW+S distance of time-series $i_1$ and $i_2$. The algorithm starts with each time-series as its own cluster, and then recursively merges clusters greedily based on the distances. We use Silhouette Coefficient~\cite{zhou2014automatic} to decide the optimal number of clusters. 
%This measure attempts to evaluate a clustering depending on how close points are to other points within the same cluster (cohesion) compared to those in the other cluster (separation). Given a point (a time-series in our case) $\mathbf{a}_{i}$ in cluster $\mathcal{C}$, its cohesion is defined as the average distance from other points in the same cluster, and its separation is defined as the minimum distance from a time-series that is in another cluster:
% \begin{equation}
%     co(\mathbf{a}_{i}) = \frac{\sum\limits_{i' \in \mathcal{C}, i' \neq i } \mathcal{D}(\mathbf{a}_{i}, \mathbf{a}_{i'})}{|\mathcal{C}-1|}\,,
%     sep(\mathbf{a}_{i}) = \min_{\mathcal{C'} \neq\mathcal{C}}\frac{\sum\limits_{i' \in \mathcal{C'}} \mathcal{D}(\mathbf{a}_{i}, \mathbf{a}_{i'})}{|\mathcal{C'}|}\,.
% \end{equation}
% The Silhouette  Coefficient of $\mathbf{a}_i \in C$ is given by
% \begin{align}
%     sil(\mathbf{a}_i) = \frac{sep(\mathbf{a}_i) - co(\mathbf{a}_i)}{\max\{sep(\mathbf{a}_i), co(\mathbf{a}_i)\}}\,&\mbox{, if } |\mathcal{C}| > 1, \nonumber\\
%     &\mbox{ otherwise, }
%     sil(\mathbf{a}_i) = 0\,.
% \end{align}
% Finally, the Silhouette Coefficient of the clustering is the average of $sil(\mathbf{a}_i)$ over all $i$. The value lies in $[-1, 1]$, and a higher value is preferred.
For classification, we use the $1$-nearest neighbor method~\cite{petitjean2014dynamic}. The choice was made so that the decisive factor in correct classification is the distance measure. That is, two time-series that are closest to each other belong to the same class. This is also a popular way of evaluating distance measures between time-series~\cite{mueen2011logical}.

\section{Experiments}
We conduct a series of experiments
to demonstrate the utility of DTW+S. Specifically, we wish to demonstrate the following:  (1) DTW+S BA leads to a more reasonable ensemble that captures the scale and timing of events. (2) DTW+S results in a more reasonable clustering compared to several other approaches.(3) DTW+S produced better classification results for many classification tasks, outperforming DTW.
In all of our experiments, unless stated otherwise, we used the following set of shapelets: (i) `increase': $[1, 2, 3, 4]$, (ii) `surge': $[1, 2, 4, 8]$, (iii) `peak': $[1, 2, 2, 1]$, and (iv) `flat': $[0, 0, 0, 0]$. According to Theorems~\ref{thm:sufficient} and~\ref{thm:necessary}, these shapelets satisfy Property~\ref{prop:2}. These were chosen because they are easily interpretable, particularly in the domain of epidemics. 
%The matrix $C_w$ constructed as in Theorem~\ref{thm:sufficient}, using this set of shapelets results in $\|C_w^{-1}\| = 13.1$, which is small when multiplied with functions of small $\epsilon_0, \epsilon_1$ as in Equation~\ref{eqn:diff_bound}.
Some other sets of shapelets that satisfy Property~\ref{prop:2} were also tried, and their results were not significantly different. 

%The flatness was calculated by setting $m_0 = 0$, and $\beta = -\ln{0.1}/\theta$, where $\theta$ is the median of the maximum ``absolute'' slope of each time-series. Recall that the ``absolute'' slope for a given window is calculated by averaging successive differences over the window. Choosing $\beta$ in such a way ensures that a window of time-series with a median ``absolute'' slope gets a low flatness of 0.1. All the code was written in MATLAB and is publicly available~\footnote{link anonymized}. 

\subsection{Ensembling}
\label{sec:exp_ens}

\begin{table}[!ht]
\centering
\footnotesize
\caption{Fractional error in estimation of peak size and timing. Darker red means higher error.}
     \label{tab:my_label}
     \setlength\tabcolsep{3pt}
\begin{tabular}{|c|c|c|c|c|c|}
\hline
                           & Method & \makecell{Mean\\ensemble}	& DTW  BA	& \makecell{DTW \\(z-norm) BA}	& \makecell{\textbf{DTW+S}\\\textbf{BA}}      \\
\hline
\multirow{7}{*}{\rotatebox[origin=c]{90}{Peak Size}} & Set 1  & -0.37\cellcolor{red!37} & -0.04 \cellcolor{red!4}  & -0.01 \cellcolor{red!1}  & -0.02 \cellcolor{red!2} \\
                           & Set 2  & -0.29\cellcolor{red!29} & -0.33 \cellcolor{red!33} & -0.29 \cellcolor{red!29} & -0.01 \cellcolor{red!1} \\
                           & Set 3  & -0.38\cellcolor{red!38} & -0.31\cellcolor{red!31}  & -0.38\cellcolor{red!38}  & -0.03\cellcolor{red!3}  \\
				& Set 4  & -0.28\cellcolor{red!28} & -0.22\cellcolor{red!22}  & -0.28\cellcolor{red!28}  & -0.01\cellcolor{red!1}  \\
				& Set 5  & -0.38\cellcolor{red!38} & -0.23\cellcolor{red!23}  & -0.38\cellcolor{red!38}  & -0.03\cellcolor{red!3}  \\
				& Set 6  & -0.28\cellcolor{red!28} & -0.21\cellcolor{red!21}  & -0.29\cellcolor{red!29}  & -0.02\cellcolor{red!2}  \\
				& Set 7  & -0.37\cellcolor{red!37} & -0.18\cellcolor{red!18}  & -0.37\cellcolor{red!37}  & -0.03\cellcolor{red!3}  \\
              & Set 8 & -0.11\cellcolor{red!11} & -0.13\cellcolor{red!13} & -0.11\cellcolor{red!11} & -0.02\cellcolor{red!2} \\
& Set 9 & -0.11\cellcolor{red!11} & -0.13\cellcolor{red!13} & -0.11\cellcolor{red!11} & -0.02\cellcolor{red!2} \\
& Set 10 & -0.12\cellcolor{red!12} & -0.14\cellcolor{red!14} & -0.12\cellcolor{red!12} & -0.02\cellcolor{red!2} \\
& Set 11 & -0.11\cellcolor{red!11} & -0.10\cellcolor{red!10} & -0.11\cellcolor{red!11} & -0.02\cellcolor{red!2} \\
& Set 12 & -0.12\cellcolor{red!11} & -0.11\cellcolor{red!11} & -0.12\cellcolor{red!12} & -0.02\cellcolor{red!2} \\

			\hline
\multirow{7}{*}{\rotatebox[origin=c]{90}{Peak Timing}}                & Set 1  & 0.08\cellcolor{red!8}   & -0.05 \cellcolor{red!5}  & -0.01 \cellcolor{red!1}  & 0.08 \cellcolor{red!8}  \\
                           & Set 2  & -0.08\cellcolor{red!8}  & -0.10 \cellcolor{red!10} & -0.08 \cellcolor{red!8}  & -0.03 \cellcolor{red!3} \\
                           & Set 3  & -0.15\cellcolor{red!15} & -0.11\cellcolor{red!11}  & -0.15\cellcolor{red!15}  & 0.07\cellcolor{red!7} \\
				& Set 4  & -0.09\cellcolor{red!9} & -0.08\cellcolor{red!8}  & -0.09\cellcolor{red!9}  & -0.03\cellcolor{red!3}  \\
				& Set 5  & 0.10\cellcolor{red!10} & -0.09\cellcolor{red!9}  & -0.10\cellcolor{red!10}  & 0.07\cellcolor{red!7}  \\
				& Set 6  & -0.09\cellcolor{red!9} & -0.08\cellcolor{red!8}  & -0.09\cellcolor{red!9}  & -0.02\cellcolor{red!2}  \\
				& Set 7  & -0.10\cellcolor{red!10} & -0.08\cellcolor{red!8}  & -0.10\cellcolor{red!10}  & 0.07\cellcolor{red!7}  \\
 & Set 8 & -0.11\cellcolor{red!11} & -0.17\cellcolor{red!17} & -0.11\cellcolor{red!11} & 0.05\cellcolor{red!5} \\
& Set 9 & -0.13\cellcolor{red!13} & -0.18\cellcolor{red!18} & -0.13\cellcolor{red!13} & 0.03\cellcolor{red!3} \\
& Set 10 & -0.13\cellcolor{red!13} & -0.18\cellcolor{red!18} & -0.13\cellcolor{red!13} & 0.02\cellcolor{red!2} \\
& Set 11 & -0.14\cellcolor{red!14} & -0.09\cellcolor{red!9} & -0.14\cellcolor{red!14} & 0.02\cellcolor{red!2} \\
& Set 12 & -0.15\cellcolor{red!15} & -0.10\cellcolor{red!10} & -0.15\cellcolor{red!15} & 0.01\cellcolor{red!1} \\
\hline  
\end{tabular}
\end{table}

\noindent\textbf{Datasets:} We consider 12 sets of trajectories, for comparing our ensemble method with baselines. These sets include the following.
% \begin{itemize}
%     \item Set 1: 75 trajectories from a model for Influenza hospitalization. The length of each trajectory is 26 (weeks). 
%     \item Set 2-7: A total of 1000 trajectories (per set) from 10 influenza models for 6 scenarios each, extracted from Influenza Scenario Modeling Hub Round 4. The length of each trajectory is 39.
%     \item Set 8-12: A total of around 1100 trajectories (per set) from 11 RSV models for six scenarios each extracted from RSV Scenario Modeling Hub Round 1. The length of each trajectory is 26 (weeks).
% \end{itemize}
\underline{Set 1}: 75 trajectories from a model for Influenza hospitalization. The length of each trajectory is 26 (weeks). 
\underline{Set 2-7}: A total of 1000 trajectories (per set) from 10 influenza models for 6 scenarios each, extracted from Influenza Scenario Modeling Hub Round 4. The length of each trajectory is 39.
\underline{Set 8-12}: A total of around 1100 trajectories (per set) from 11 RSV models for six scenarios each extracted from RSV Scenario Modeling Hub Round 1. The length of each trajectory is 26 (weeks).

\noindent\textbf{Methods:}
We compare the following approaches for ensembling. 
% \begin{itemize}
%      \item Mean ensemble: the most popular ensemble approach that simply averages values at each time-point~\cite{SMH,ESMH}
%      \item DTW BA: barycenter averaging with DTW
%      \item DTW (z-norm) BA: barycenter averaging with DTW after applying z-normalization to all trajectories
%      \item DTW+S BA: barycenter averaging with DTW+S.
% \end{itemize}
\underline{Mean ensemble}: the most popular ensemble approach that simply averages values at each time-point~\cite{SMH,ESMH};
\underline{DTW BA}: barycenter averaging with DTW;
\underline{DTW (z-norm) BA}: barycenter averaging with DTW after applying z-normalization to all trajectories; and
\underline{DTW+S BA}: barycenter averaging with DTW+S.

\begin{figure}[!ht]
    \centering
    \subfigure[Ensemble comparison]{
    \includegraphics[trim={0.8cm 0.1cm 0.8cm 0.2cm},clip,width=0.6\columnwidth]{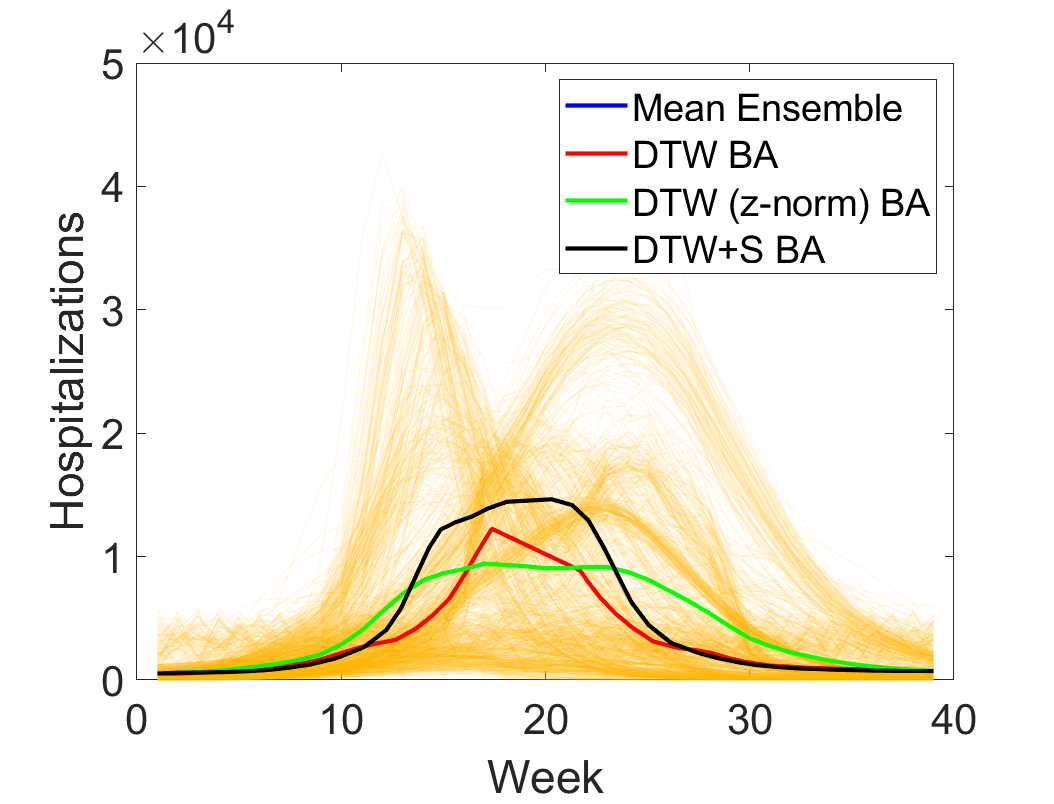}
    \label{fig:ens_set7}
    }
    \subfigure[Intermediate step]{
    \includegraphics[trim={0.2cm 0.1cm 0.8cm 0.2cm},clip,width=0.47\columnwidth]{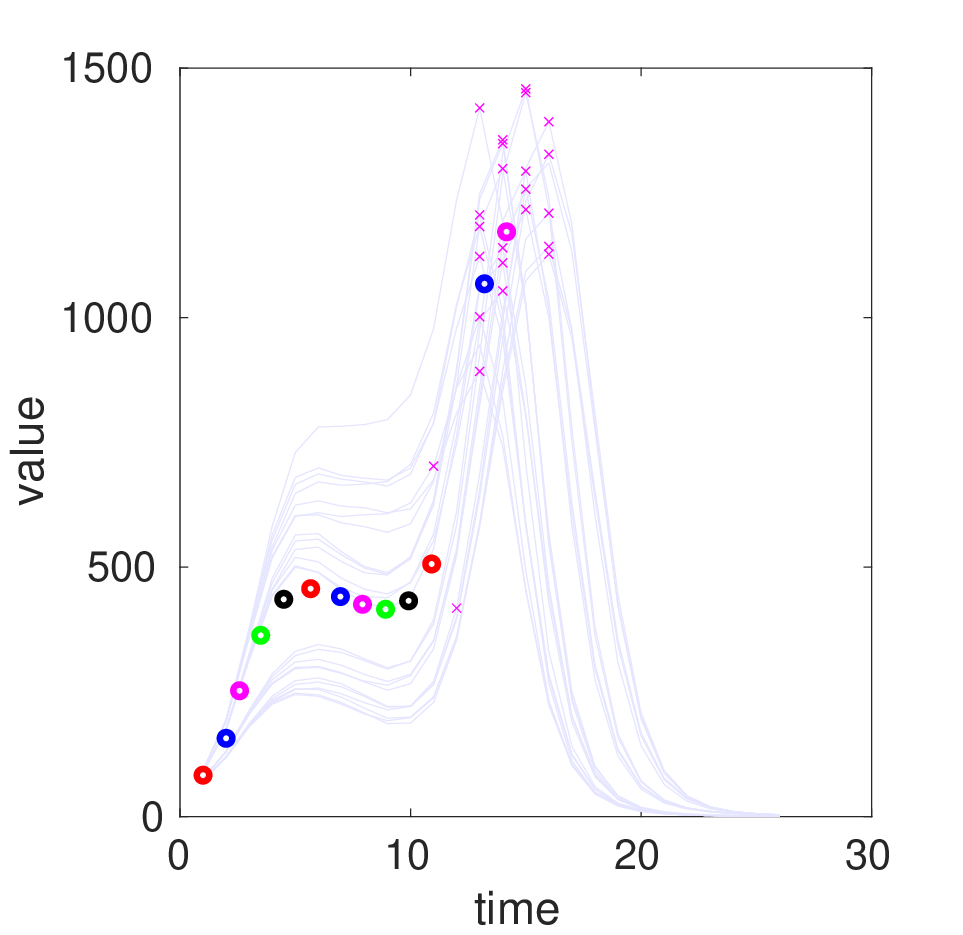} 
    \label{fig:ensemble_interm}
    }
    \caption{Ensembling results: (a) Different ensembling approaches on Set 7. The dark yellow lines represent the individual trajectories. (b) An instance of alignment on Set 1. Pink `x' on the individual time-series are aligned to get the ensemble point (pink circle). Previous circles represent the ensemble points obtained from previous alignments}
    \label{fig:ensemble-process}
\end{figure}

The resulting ensembles for Set 7 are shown in Figure~\ref{fig:ens_set7}. We observe that DTW+S BA produces the highest peak among all ensembles. Mean ensemble and DTW (z-norm) BA produce almost identical results (the green line overlaps with the blue line) and flatten the peak. To quantitatively compare the ensembles, we measure the fractional error of the ensembles in representing the peak timing and the peak size (i.e., the value and the time) at which the peak occurs. The ground truth is obtained by extracting the peak values (and timing) of all trajectories and averaging them. The results are presented in Table~\ref{tab:my_label}. Note that our approach is the only one that captures the peak timing and size of the underlying trajectories well for all sets.

\begin{figure*}[!ht]
    \centering
    %\subfigure[]{%
        \includegraphics[width=0.30\textwidth,trim=1.5cm 0.2cm 1.2cm 0.2cm,clip]{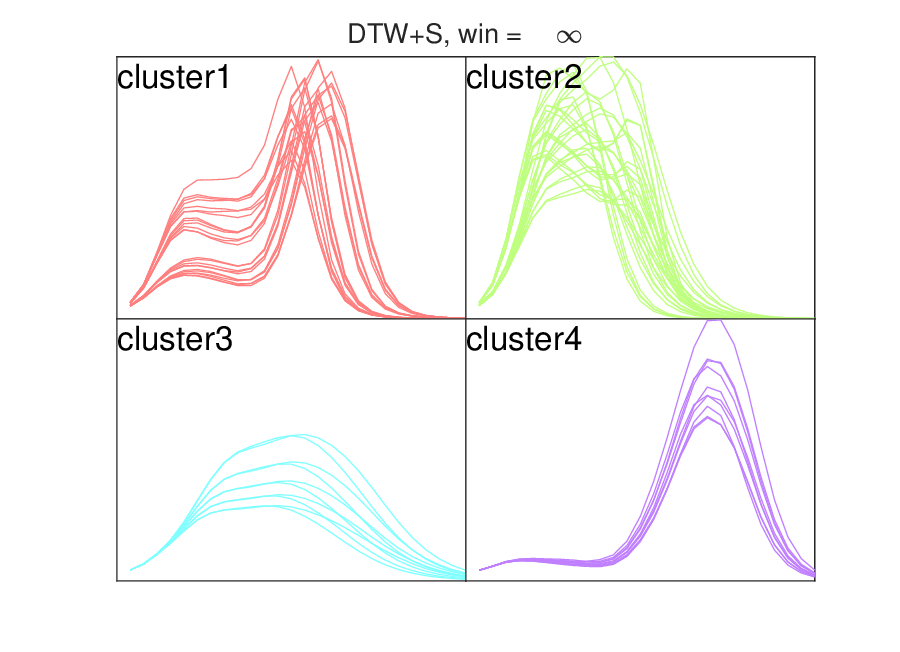}
    %}
    \hfill
    % \subfigure[]{%
        \includegraphics[width=0.30\textwidth,trim=1.5cm 0.2cm 1.2cm 0.2cm,clip]{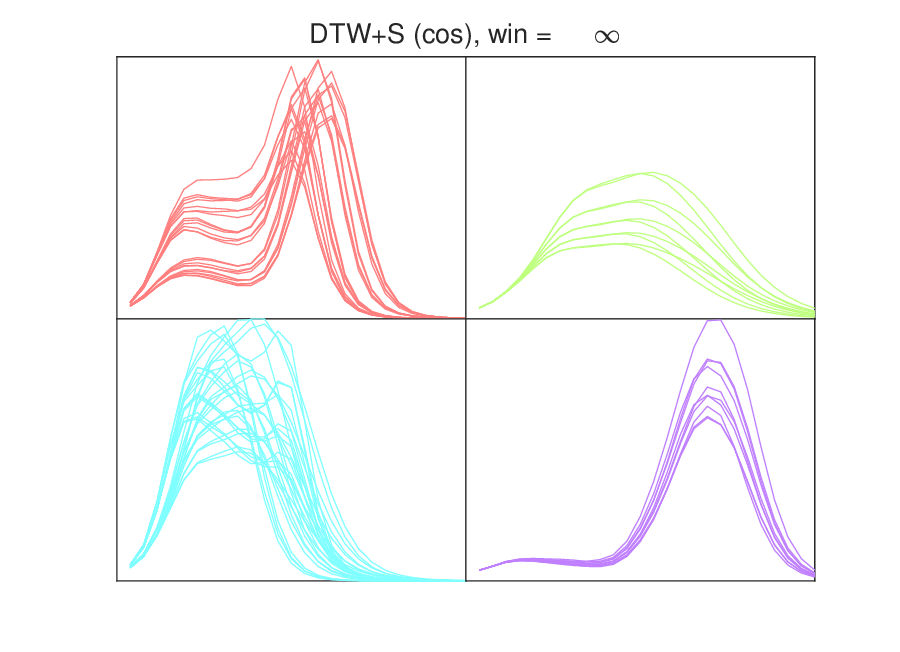}
    % }
    \hfill
    %\subfigure[]{%
        \includegraphics[width=0.30\textwidth,trim=1.5cm 0.2cm 1.2cm 0.2cm,clip]{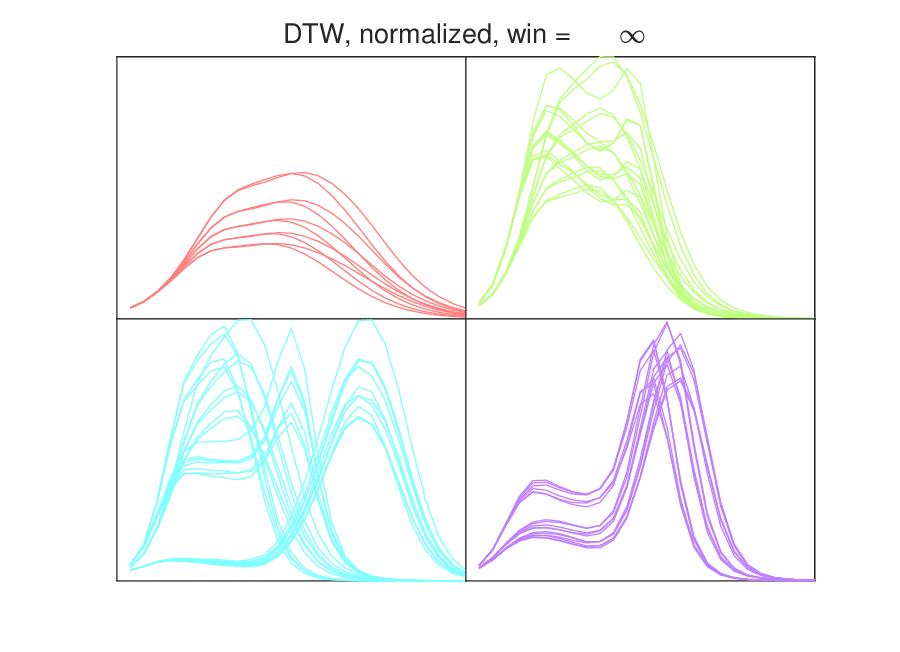}
    %}
    \hfill
    % \subfigure[]{%
        \includegraphics[width=0.30\textwidth,trim=1.5cm 0.2cm 1.2cm 0.2cm,clip]{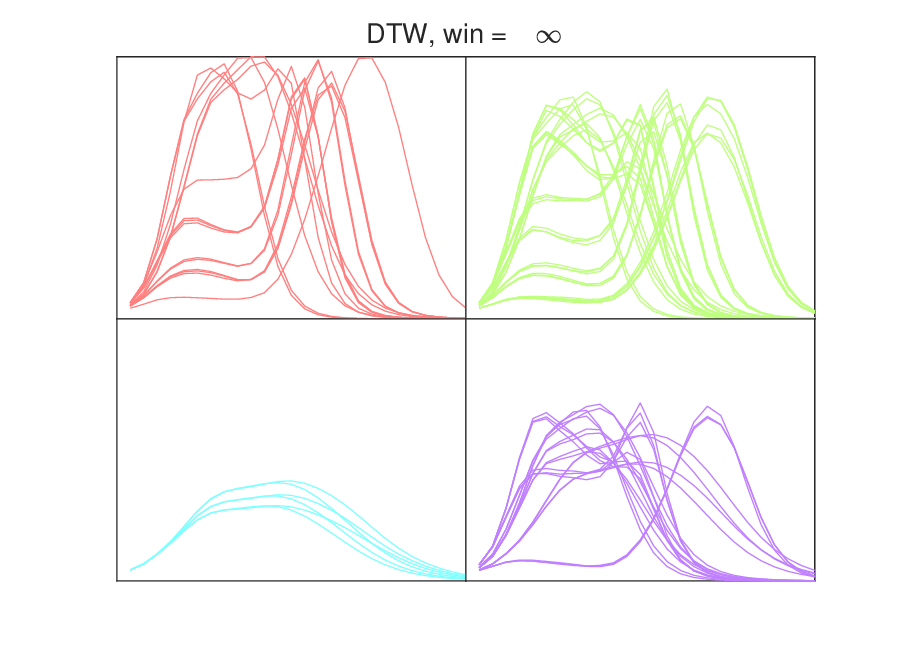}
    % }
    \hfill
    % \subfigure[]{%
      \includegraphics[width=0.30\textwidth,trim=1.5cm 0.2cm 1.2cm 0.2cm,clip]{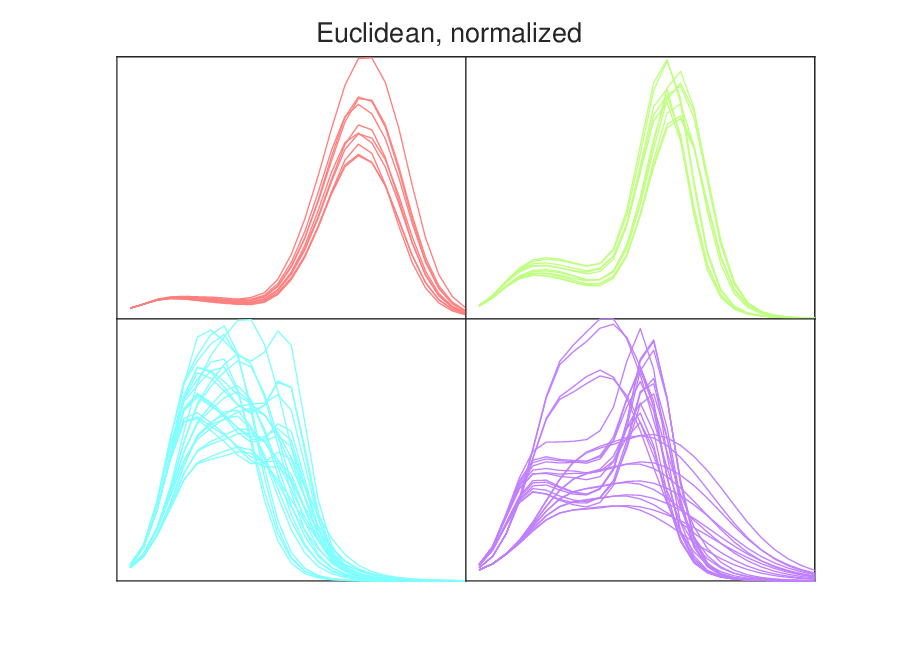}
    % }
    \hfill
    %\subfigure[]{%
        \includegraphics[width=0.30\textwidth,trim=1.5cm 0.2cm 1.2cm 0.2cm,clip]{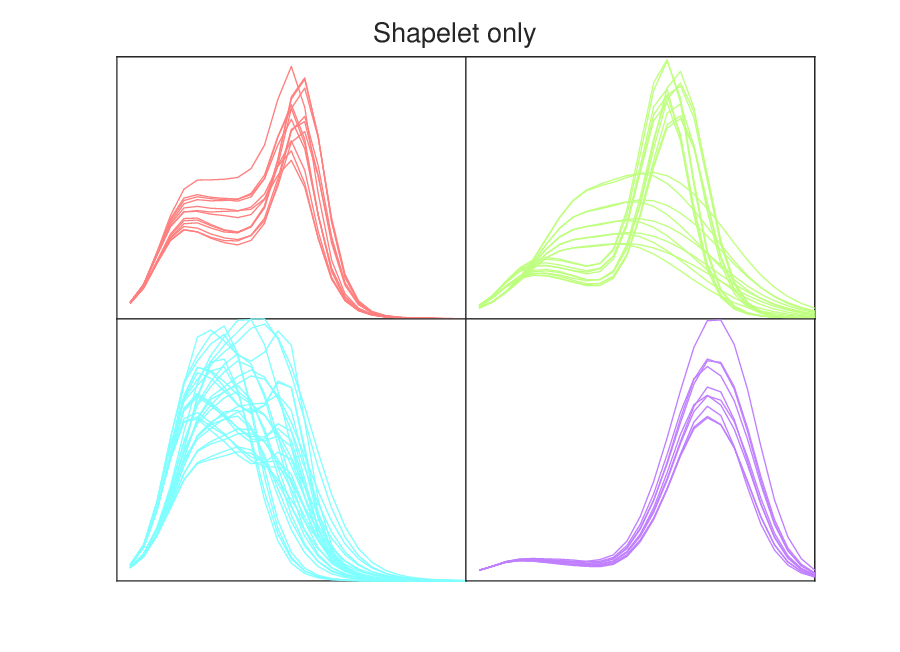}
    %}
    %\vspace{-0.25cm}
    \caption{Comparison of clustering results obtained from DTW+S against other distance measures. DTW+S with Euclidean (default) and cosine distance produces reasonable clustering. All other measures mix different patterns into one cluster.}
    \label{fig:clustering_results}
\end{figure*}

\subsection{Clustering: A Qualitative Evaluation}
\label{sec:exp_cluster}
We consider time-series projections for weekly influenza hospitalization for a US state by a model from Influenza Scenario Modeling Hub~\cite{fluSMH}. It has 75 time-series, each corresponding to different choices of parameters and initialization.  We calculated the dissimilarity matrix (all pair distances) using the following. (1) DTW+S: our method with infinite window for alignment; (2): DTW+S (cos): same as DTW+S, except that cosine distance is used instead of Euclidean  for aligning SSRs; (3) DTW, normalized: Applying DTW after normalizing all time-series to zero mean and unit variance -- a common normalization technique used with DTW~\cite{UCRArchive}; (4) DTW: DTW without any transformation or normalization; (5) Euclidean, normalized: Euclidean distance without any time warping on standard normal time-series; and (6) Shapelet-only: Euclidean distance on the SSR without any time-warping. For DTW+S, we generate hierarchical clusters with the number of clusters selected using the Silhouette Coefficient as described in Section~\ref{sec:class_and_clust}. We use the same number of clusters for clustering using each of the above dissimilarity measures. 

Figure~\ref{fig:clustering_results} shows the results of clustering. We observe that DTW+S produces four clusters with similarly shaped time-series. DTW with standard normalization mixes clusters 1, 2, and 4. The Shapelet-only measure mixes clusters 1 and 3. DTW+S (cos) produces the same clustering (different ordering); DTW without normalization does not produce any discernable pattern in the trends and instead seems to group together those time-series that have similar peak heights; Euclidean distance with standard normalization mixes clusters 1, 2 and 4. Note that Shapelet transformation is not sufficient to capture similar time-series due to not being flexible across time. On the other hand, DTW with simple standard normalization cannot capture similar trends that occur at different scales. However, when they are combined in DTW+S, they produce reasonable clustering.

\subsection{Classification}
\label{sec:exp_classify}
For classification, we use 64 datasets available at UCR Time Series Classification 2015 Archive~\cite{UCRArchive}. Since DTW+S is based on the distance between local trends, we ignored very large time-series where trends would be difficult to infer visually. Handling long time-series may require some sampling and using larger shapelets. Further, an analysis of our approach on multiple large datasets will require more optimized implementation. These analyses and optimizations are beyond the scope of this work. Therefore, we removed all the datasets where the length of the time-series was more than 800 and the number of training data instances where more than 700, giving us 64 datasets out of 86.

\begin{figure}[!ht]
    \centering
    \includegraphics[width=0.9\columnwidth, trim={310 0 80 0}, clip]{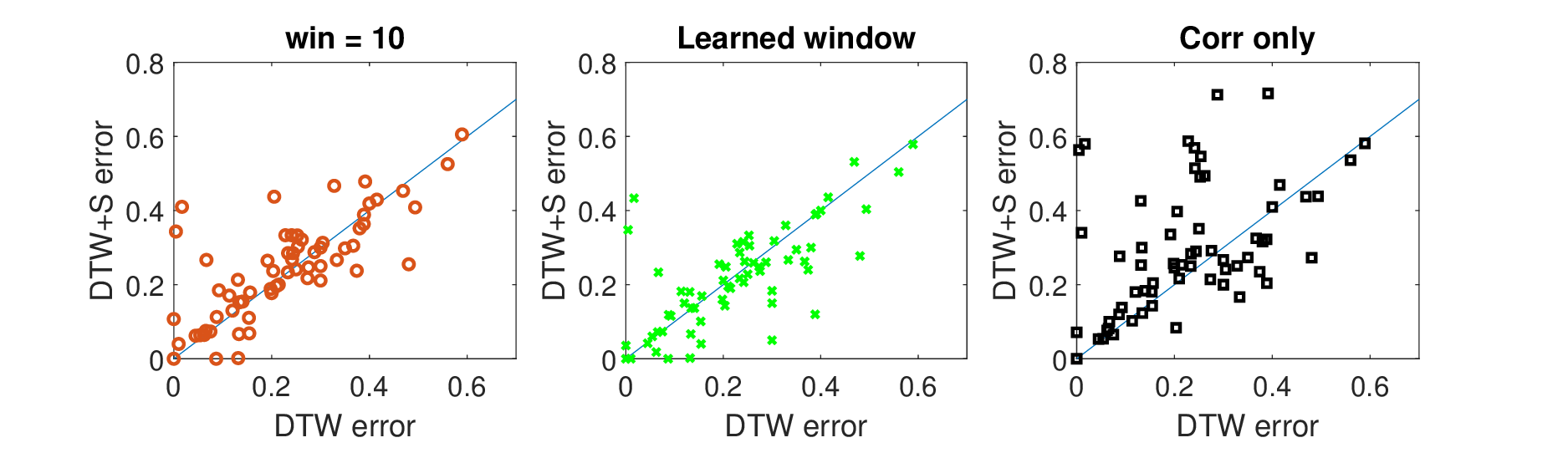}
    \caption{Performance of DTW+S on 64 datasets.}
    \label{fig:val_results}
\end{figure}

\begin{figure*}[!htb]
    \centering
    \subfigure[better (0.0017) vs DTW (0.13)]{
        \includegraphics[width=0.26\textwidth]{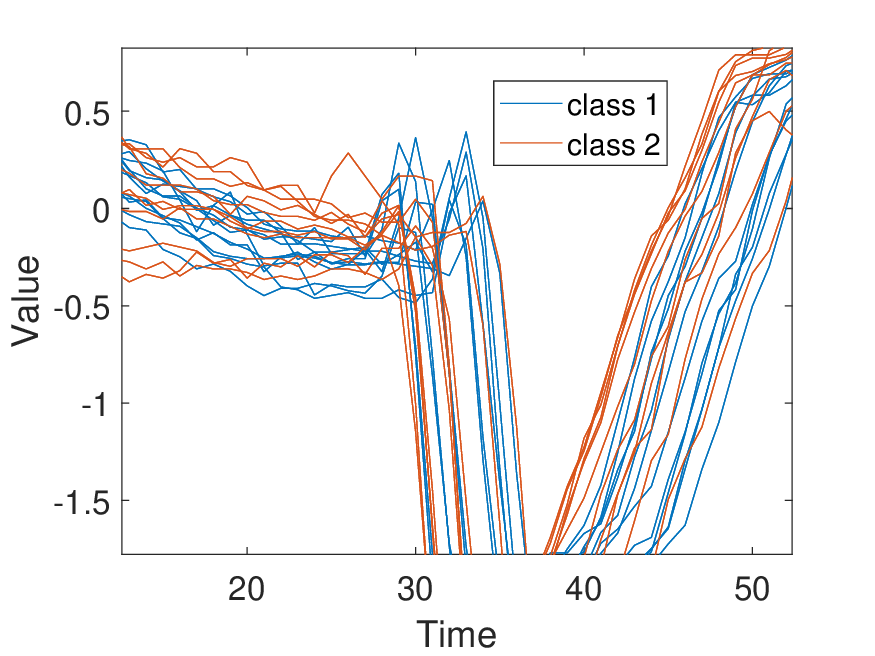}    
        %\caption{better (0.0017) vs DTW (0.13)}
        \label{fig:good}
    }
    \subfigure[worse (0.23) vs DTW (0.067)]{
        \includegraphics[width=0.26\textwidth]{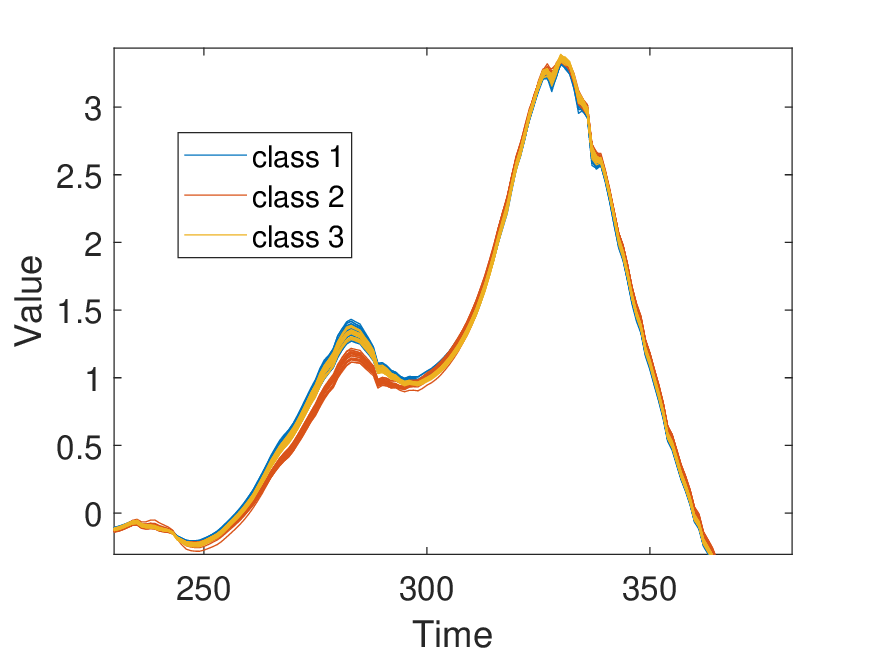}    
        %\caption{worse (0.23) vs DTW (0.067)}
        \label{fig:bad}
   }
    \subfigure[worse (0.35) vs DTW (0.0044)]{
        \includegraphics[width=0.26\textwidth]{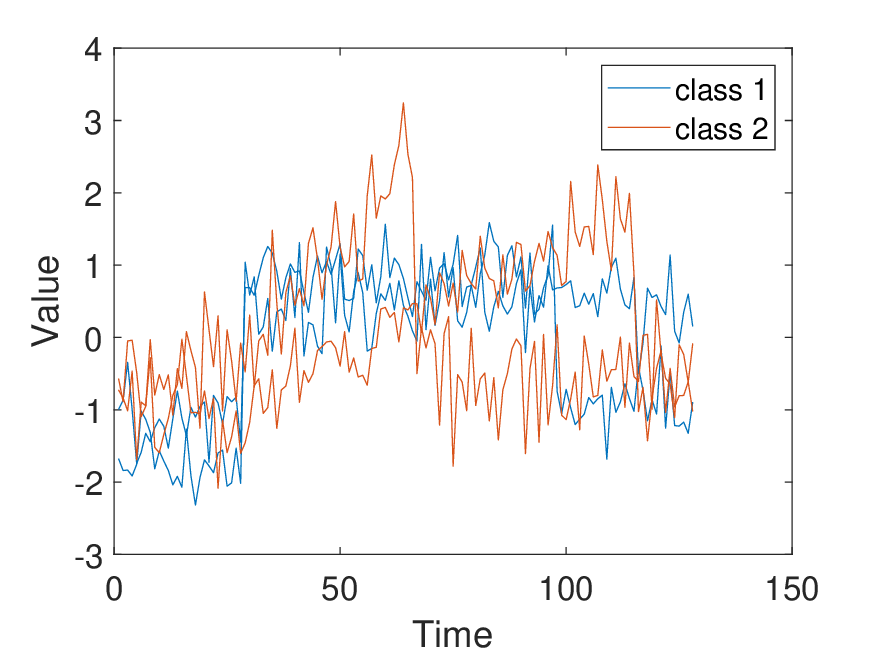}    
        %\caption{worse (0.35) vs DTW (0.0044)}
        \label{fig:noisy}
    }
    \caption{Section of time-series color-coded to show the different classes. (a) DTW+S picks up small local trends that are not captured by DTW alone. (b) The difference in the classes is the scale rather than the shape in certain parts of the time-series, making it not a desirable dataset for DTW+S. (c) The time-series possess noise preventing DTW+S from identifying trends.}
    \label{fig:good_bad}
\end{figure*}

\begin{figure}[!ht]
    \centering
    \includegraphics[width=0.85\columnwidth, trim={310 0 80 0}, clip]{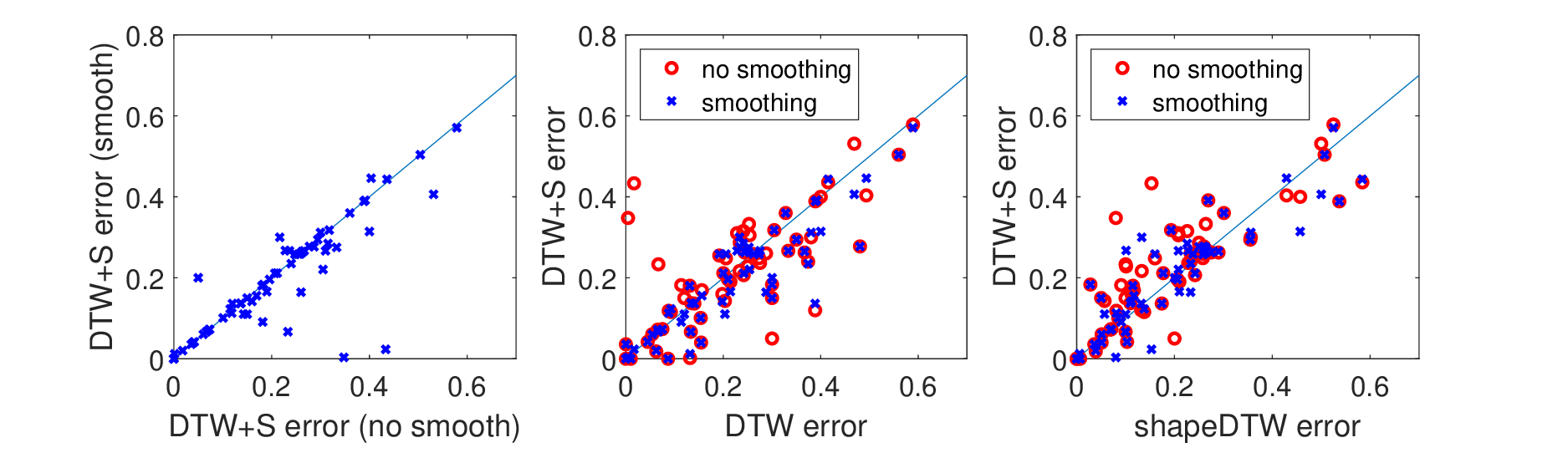}
    \caption{Results obtained by DTW+S with smoothing. (a) Improvement on many datasets where DTW+S was worse than DTW. (b) Comparison against shapeDTW.}
    \label{fig:smooth_results}
\end{figure}

For each dataset, we find the 1-nearest neighbor for each instance of test time-series in the training set using DTW+S and assign its class. Then, we evaluate our approach using error defined as the fraction of misclassification, which is the commonly used evaluation method for these datasets~\cite{petitjean2014dynamic}. We treat the warping window $\tau$ as a hyperparameter. 
%For one set of results, we arbitrarily set $\tau = 10$. For another set of results 
We use leave-one-out cross-validation to identify $\tau$ as a fraction of the length of the time-series $T$ from a set of values -- $\tau = \{0, 0.01T, 0.02T, \dots, 0.07T\}$. We compare our results with the reported best performance of DTW~\cite{UCRArchive}. The results are presented in Figure~\ref{fig:val_results}. In the scatter plot, each point represents errors on a dataset. The blue line represents the line $y=x$, i.e., when DTW+S has the same error as that of DTW. A point lying below (and to the right) of this line indicates that DTW+S error was lower than DTW error.
We observe that 
%DTW+S outperforms DTW on 45.3\% datasets.  The two measures lead to similar errors for many datasets (along the $y=x$ line). With the windows learned through leave-one-out validation the errors improve significantly on many datasets as 
DTW+S outperforms DTW on 57.7\% datasets.  The two measures lead to similar errors for many datasets (along the $y=x$ line). ``Corr only'' represents the DTW+S measure obtained by ignoring the ``flat'' shapelet. In this case, DTW+S reduces to the set of Person correlations with the three other shapelets. 
%Now, even with the learned warping window, 
DTW+S (corr) only outperforms DTW in 39.1\% datasets. Additionally, it produces some large errors. Correlation ignores scale completely, and small fluctuations that are not necessarily useful patterns, but noise can cause the measure to consider it similar to some other significant patterns. Thus, the \textit{flatness dimension has a significant contribution} to the performance of DTW+S.

We note that DTW+S does not outperform DTW on all datasets. This is expected as DTW+S focuses more on the shapes rather than the scale. Figure~\ref{fig:good_bad} shows examples where DTW+S and DTW have significantly different performances. For the dataset in Figure~\ref{fig:good}, DTW+S has an error of 0.0017 while DTW has an error of 0.13. This is because DTW+S picks up small local trends (a spike around time 30) present in one class and absent in another class. This is not captured by DTW alone. For the dataset in Figure~\ref{fig:bad}, DTW+S has an error of 0.23 while DTW has a much lower error of 0.07. For this dataset, the difference in the classes is the scale rather than the shape in certain parts of the time-series. DTW is able to identify this distinction, making it a less desirable dataset for DTW+S. 

\subsubsection{Smoothing}

Another type of dataset that would be undesirable for DTW+S is that where the time-series have high noise (Figure~\ref{fig:noisy}). This noise will impact the identification of $\beta$ for the flatness parameter and the identification of local trends. One way to address this is to smooth the time-series before finding the Shapelet-space Transformation. While a domain expert may choose a reasonable method for smoothing, we use a moving average method with the window size chosen with leave-one-out cross-validation from the set $\{0, 0.1T, 0.2T, 0.4T\}$. The results are shown in Figure~\ref{fig:smooth_results}. The first plot shows that allowing smoothness generally improves the error (many datasets fall to the bottom right). For a very small number of datasets validation results seem to pick a smoothing window that results in worse performance on the test set. In practice, this could be mitigated by having a larger training set. The second plot of Figure~\ref{fig:smooth_results} shows that the smoothing significantly brings down the error for some datasets (e.g., the three circles on the left of the plot drop close to zero). As an example, the dataset in Figure~\ref{fig:noisy} for which DTW+S was significantly worse (error of 0.35) compared to DTW (0.0044), smoothing results in an error dropping to 0.0022. Finally, the third plot in Figure~\ref{fig:smooth_results} compares our approach against shapeDTW~\cite{zhao2018shapedtw}. We select ``HOG1D'' version of shapeDTW as it performs the best for this collection of datasets. Without smoothing, there are a few datasets where DTW+S is much worse than shapeDTW (higher in the plot). After smoothing, most datasets accumulate around $y=x$ line. 

\section{Discussion}\label{sec:explain}
\paragraph{Interpretability} 
A key advantage of our approach is interpretability. Since the SSR is determined by the similarity of the given trend with respect to pre-defined shapes of interest, one can easily make sense of the representation. This is particularly useful for application domains where there exist certain shapes of interest (e.g., increase and peak in public health) and there is resistance towards adopting black-box approaches. Figure~\ref{fig:explanation} shows the SSR of two samples each from two classes of a dataset. This is the dataset corresponding to Figurer~\ref{fig:good}, where we observed small peaks appearing in one of the classes. Based on the SSR of the samples, we observe that samples for class 1 show a bright yellow (high value) corresponding to dimension 2. On the other hand, samples from class 2 have a dull yellow (lower value) for the same dimension. This dimension corresponds to the shapelet ``peak'' $=[1, 2, 2, 1]$, thus suggesting that a peak around the $t=25$ makes sample 1 in class 1 more similar to sample 2 in class 1 compared to the samples in another class.
\begin{figure}[!ht]
     \centering
     \includegraphics[width=0.92\columnwidth, trim={0 5 0 0}, clip]{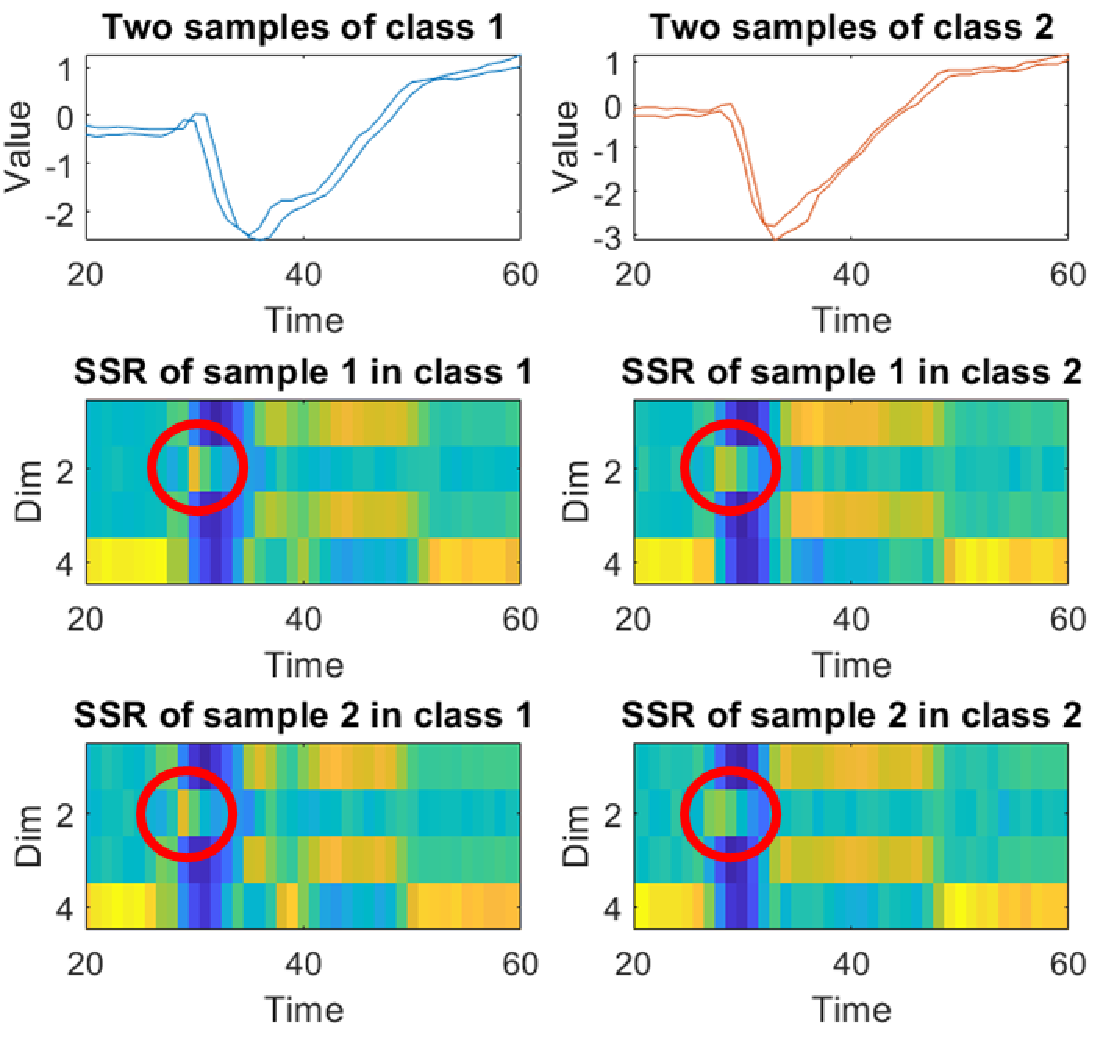}
     \caption{Interpreting the Shapelet Space Representation.}
     \label{fig:explanation}
     \vspace{-0.3cm}
 \end{figure}

 \paragraph{Limitations}
Our approach is \textit{not designed for general-purpose classification} or encoding, particularly, where shapes have little impact compared to the scale. Furthermore, while warping windows and smoothing parameters can be set through validation, the best utilization of DTW+S would require some domain knowledge to understand their appropriate setting and the choice of shapelets. However, Theorems~\ref{thm:sufficient} and~\ref{thm:necessary} act as guidelines to ensure that the chosen shapelets satisfy the desired property of closeness preservation. Another limitation is the implementation -- currently, we use $\mathcal{O}(T^2)$ time and space for DTW on the distance matrix obtained from SSR. In future work, we will explore existing optimizations of DTW~\cite{keogh2000scaling}.

\section{Conclusion}
We have proposed a novel interpretable distance measure for time-series that looks for a sequence of similar trends occurring around the same time. 
It can capture local trends in a representation that is closeness preserving. We have demonstrated that our approach DTW+S, which applies DTW on our SSR matrices, results in better clustering, which cannot be achieved by DTW or SSR alone. We have developed an ensemble method using DTW+S that captures both the aggregate scale and timing of the individual time-series significantly better than the currently used mean ensemble. We have shown that DTW+S can result in better classification compared to other measures for a large number of datasets, particularly those where local trends play a key role.

\section*{Acknowledgements}
This work was supported by the Centers for Disease Control and Prevention and the National Science Foundation under the awards no. 2135784, 2223933, and 2333494. Any opinions, findings, and conclusions or recommendations expressed in this material are those of the author and do not necessarily reflect the views of the National Science Foundation and the Center for Disease Control and Prevention.
We would like to thank the US and European Scenario Modeling Hub for useful discussions.

\bibliography{sample,epidemic}

\begin{thebibliography}{23}
\providecommand{\natexlab}[1]{#1}

\bibitem[{Borchering(2023)}]{Flusight2023}
Borchering, R. 2023.
\newblock FluSight 2023-2024.
\newblock \url{https://github.com/cdcepi/FluSight-forecast-hub}.

\bibitem[{Chen et~al.(2015)Chen, Keogh, Hu, Begum, Bagnall, Mueen, and
  Batista}]{UCRArchive}
Chen, Y.; Keogh, E.; Hu, B.; Begum, N.; Bagnall, A.; Mueen, A.; and Batista, G.
  2015.
\newblock The UCR Time Series Classification Archive.
\newblock \url{www.cs.ucr.edu/~eamonn/time_series_data/}.

\bibitem[{Day and Edelsbrunner(1984)}]{day1984efficient}
Day, W.~H.; and Edelsbrunner, H. 1984.
\newblock Efficient algorithms for agglomerative hierarchical clustering
  methods.
\newblock \emph{Journal of classification}, 1(1): 7--24.

\bibitem[{Dhamo, Ismailaja, and Kallu{\c{c}}i(2015)}]{dhamo2015comparing}
Dhamo, E.; Ismailaja, N.; and Kallu{\c{c}}i, E. 2015.
\newblock Comparing the efficiency of CID distance and CORT coefficient for
  finding similar subsequences in time series.
\newblock In \emph{Sixth International Conference ISTI}, 5--6.

\bibitem[{Divas{\'o}n and Aransay(2013)}]{divason2013rank}
Divas{\'o}n, J.; and Aransay, J. 2013.
\newblock Rank-nullity theorem in linear algebra.
\newblock \emph{Archive of Formal Proofs}.

\bibitem[{{European CDC}(2022)}]{ESMH}
{European CDC}. 2022.
\newblock {European COVID-19 Scenario Hub}.
\newblock
  \url{https://github.com/covid19-forecast-hub-europe/covid19-scenario-hub-europe}.

\bibitem[{Gupta et~al.(1996)Gupta, Molfese, Tammana, and
  Simos}]{gupta1996nonlinear}
Gupta, L.; Molfese, D.~L.; Tammana, R.; and Simos, P.~G. 1996.
\newblock Nonlinear alignment and averaging for estimating the evoked
  potential.
\newblock \emph{IEEE transactions on biomedical engineering}, 43(4): 348--356.

\bibitem[{Howerton et~al.(2023)Howerton, Contamin, Mullany, Qin, Reich, Bents,
  Borchering, Jung, Loo, Smith et~al.}]{howerton2023informing}
Howerton, E.; Contamin, L.; Mullany, L.~C.; Qin, M.; Reich, N.~G.; Bents, S.;
  Borchering, R.~K.; Jung, S.-m.; Loo, S.~L.; Smith, C.~P.; et~al. 2023.
\newblock Informing pandemic response in the face of uncertainty. An evaluation
  of the US COVID-19 Scenario Modeling Hub.
\newblock \emph{medRxiv}.

\bibitem[{Jeong and Jayaraman(2015)}]{jeong2015support}
Jeong, Y.-S.; and Jayaraman, R. 2015.
\newblock Support vector-based algorithms with weighted dynamic time warping
  kernel function for time series classification.
\newblock \emph{Knowledge-based systems}, 75: 184--191.

\bibitem[{Keogh and Pazzani(2000)}]{keogh2000scaling}
Keogh, E.~J.; and Pazzani, M.~J. 2000.
\newblock Scaling up dynamic time warping for datamining applications.
\newblock In \emph{Proceedings of the sixth ACM SIGKDD international conference
  on Knowledge discovery and data mining}, 285--289.

\bibitem[{Lines and Bagnall(2015)}]{lines2015time}
Lines, J.; and Bagnall, A. 2015.
\newblock Time series classification with ensembles of elastic distance
  measures.
\newblock \emph{Data Mining and Knowledge Discovery}, 29: 565--592.

\bibitem[{Mueen, Keogh, and Young(2011)}]{mueen2011logical}
Mueen, A.; Keogh, E.; and Young, N. 2011.
\newblock Logical-shapelets: an expressive primitive for time series
  classification.
\newblock In \emph{Proceedings of the 17th ACM SIGKDD international conference
  on Knowledge discovery and data mining}, 1154--1162.

\bibitem[{M{\"u}ller(2007)}]{muller2007dynamic}
M{\"u}ller, M. 2007.
\newblock Dynamic time warping.
\newblock \emph{Information retrieval for music and motion}, 69--84.

\bibitem[{Petitjean et~al.(2014)Petitjean, Forestier, Webb, Nicholson, Chen,
  and Keogh}]{petitjean2014dynamic}
Petitjean, F.; Forestier, G.; Webb, G.~I.; Nicholson, A.~E.; Chen, Y.; and
  Keogh, E. 2014.
\newblock Dynamic time warping averaging of time series allows faster and more
  accurate classification.
\newblock In \emph{2014 IEEE international conference on data mining},
  470--479. IEEE.

\bibitem[{Petitjean, Ketterlin, and
  Gan{\c{c}}arski(2011)}]{petitjean2011global}
Petitjean, F.; Ketterlin, A.; and Gan{\c{c}}arski, P. 2011.
\newblock A global averaging method for dynamic time warping, with applications
  to clustering.
\newblock \emph{Pattern recognition}, 44(3): 678--693.

\bibitem[{Ratanamahatana and Keogh(2004)}]{ratanamahatana2004making}
Ratanamahatana, C.~A.; and Keogh, E. 2004.
\newblock Making time-series classification more accurate using learned
  constraints.
\newblock In \emph{Proceedings of the 2004 SIAM international conference on
  data mining}, 11--22. SIAM.

\bibitem[{Srivastava, Singh, and Lee(2022)}]{srivastava2022shape}
Srivastava, A.; Singh, S.; and Lee, F. 2022.
\newblock Shape-based Evaluation of Epidemic Forecasts.
\newblock In \emph{2022 IEEE International Conference on Big Data (Big Data)},
  1701--1710. IEEE.

\bibitem[{{US SMH}({2020})}]{SMH}
{US SMH}. {2020}.
\newblock {COVID}-19 {Scenario} {Modeling} {Hub}.
\newblock https://github.com/midas-network/covid19-scenario-modeling-hub.

\bibitem[{{US SMH}(2022)}]{fluSMH}
{US SMH}. 2022.
\newblock {Flu Scenario Modeling Hub}.
\newblock \url{https://fluscenariomodelinghub.org/}.

\bibitem[{Van~Fleet(2019)}]{van2019discrete}
Van~Fleet, P.~J. 2019.
\newblock \emph{Discrete wavelet transformations: An elementary approach with
  applications}.
\newblock John Wiley \& Sons.

\bibitem[{Ye and Keogh(2009)}]{ye2009time}
Ye, L.; and Keogh, E. 2009.
\newblock Time series shapelets: a new primitive for data mining.
\newblock In \emph{Proceedings of the 15th ACM SIGKDD international conference
  on Knowledge discovery and data mining}, 947--956.

\bibitem[{Zhao and Itti(2018)}]{zhao2018shapedtw}
Zhao, J.; and Itti, L. 2018.
\newblock shapedtw: Shape dynamic time warping.
\newblock \emph{Pattern Recognition}, 74: 171--184.

\bibitem[{Zhou and Gao(2014)}]{zhou2014automatic}
Zhou, H.~B.; and Gao, J.~T. 2014.
\newblock Automatic method for determining cluster number based on silhouette
  coefficient.
\newblock In \emph{Advanced materials research}, volume 951, 227--230. Trans
  Tech Publ.

\end{thebibliography}

\clearpage

\appendix
\section*{Supplementary Material}
\section{Proofs}
\begin{theorem}
Property~\ref{prop:2} is satisfied with any set of $w-1$ linearly independent shapelets and the ``flat'' shapelet, i.e., with this choice $\|f(\mathbf{x}) - f(\mathbf{y})\| \leq \epsilon$ iff (i) both $\mathbf{x}$ and $\mathbf{y}$ are ``almost'' flat, or (ii) $\|\mathbf{x'} - \mathbf{y'}\| \leq \delta$, for some small $\epsilon$ and $\delta$.
\end{theorem}

\begin{proof}
Suppose $\|f(\mathbf{x}) - f(\mathbf{y})\| \leq \epsilon$. 

If both vectors are approximately flat then, without loss of generality, we can assume that $\phi_x \geq \phi_y \geq 1-\varepsilon$, for some small $\varepsilon$. Then, along the flat dimension, $|2(\phi_x -1) - 2(\phi_y-1)|^2 \leq 4(1 - (1 - \varepsilon))^2 \leq 4\varepsilon^2$. And, along any other dimension,
\begin{align*}
    |(1-\phi_x)\mathbf{s'}^T\mathbf{x'} &- (1-\phi_y)\mathbf{s'}^T\mathbf{y'}|^2 \\
    &\leq |(1-\phi_y)\cdot 1 - (1-\phi_x)\cdot (-1)|^2 \\
    &= |2 - (\phi_x + \phi_y)|^2 \leq |2 - 2(1 - \varepsilon)|^2 \leq 4 \varepsilon^2.
\end{align*}
So, $\|f(\mathbf{x}) - f(\mathbf{y})\|^2 \leq 4(w-1)\varepsilon^2 + 4\varepsilon = 4w\varepsilon^2 = \epsilon^2$, where $\varepsilon = \epsilon/(2\sqrt{w})$. 

Now, suppose that both vectors are not ``almost'' flat and $\phi_x \geq \phi_y$. Let $\epsilon^2 = \sum_i \epsilon^2_i$, where $\epsilon_i$ is the difference in the $i^{th}$ dimension of $f(\mathbf{x}) -f(\mathbf{y})$.
Then, along the dimension corresponding to $\mathbf{s}_0$:
\begin{equation}
    |(2\phi_x - 1) - (2\phi_y - 1)| \leq \epsilon_0
    \implies \phi_\mathbf{x} \leq \phi_\mathbf{y} + \epsilon_0/2.
\end{equation}
Now, since  $\phi_y \leq \phi_x \leq \phi_y + \epsilon_0/2$, then a small $\phi_y$ would imply that $\phi_x$ is also small. Therefore, both $\phi_x$ and $\phi_y$ are not small.

Now, along any other dimension $i$,
\begin{align*}
\epsilon_i &= |(1-\phi_x)\mathbf{s'}^T\mathbf{x'} - (1-\phi_y)\mathbf{s'}^T\mathbf{y'} | \\
&= |(1-\phi_x)\mathbf{s'}^T(\mathbf{x'} - \mathbf{y'}) + (1-\phi_x)\mathbf{s'}^T\mathbf{y'} - (1-\phi_y)\mathbf{s'}^T\mathbf{y'} | \\
&= |(1-\phi_x)\mathbf{s'}^T(\mathbf{x'} - \mathbf{y'}) + (\phi_y -\phi_x)\mathbf{s'}^T\mathbf{y'}| \\
&\geq |(1-\phi_x)\mathbf{s'}^T(\mathbf{x'} - \mathbf{y'})| -  \epsilon_0/2\,.
\end{align*}
And thus,
\begin{align}
    |(1-\phi_x)\mathbf{s'}^T(\mathbf{x'} - \mathbf{y'})| &\leq \epsilon_i + \epsilon_0/2 \nonumber\\
    \implies 
    |\mathbf{s'}^T(\mathbf{x'} - \mathbf{y'})| &\leq \frac{\epsilon_i + \epsilon_0/2}{1 - \phi_x}\,,\forall i\in\{1 \dots w-1\}\,.
\end{align}

% of $f(\mathbf{x})$ is given by $(1-\phi_\mathbf{x})corr(\mathbf{s}_i, \mathbf{x}) = (1-\phi_\mathbf{x}) \mathbf{s'_i}^T \mathbf{x'}$. Since this is equal to the dimension $i$ of $\mathbf{y}$, we have
% \begin{align}
%     (1-\phi_\mathbf{x}) \mathbf{s'_i}^T \mathbf{x'} &= (1-\phi_\mathbf{y}) \mathbf{s'_i}^T \mathbf{y'} \nonumber \\
%     \implies (1-\phi_\mathbf{y}) \mathbf{s'_i}^T \mathbf{x'} &= (1-\phi_\mathbf{y}) \mathbf{s'_i}^T \mathbf{y'} \\
%     \implies \mathbf{s'_i}^T(\mathbf{y'} - \mathbf{x'}) &= 0, \forall i>0.
% \end{align}
% Now, we identify the values that can be taken by $\mathbf{y'} - \mathbf{x'}$. 
Additionally, we note that due to the normalization, 

\begin{equation}
    \mathbf{1}^T\mathbf{y'} = \mathbf{1}^T\mathbf{x'} = 0 \implies \mathbf{1}^T(\mathbf{y'} - \mathbf{x'}) = 0\,,
\end{equation}
where $\mathbf{1}$ is the vector of all ones.
Now consider a matrix $C_w$ with $w$ rows whose $i^{th}$ row is given by $\mathbf{s'}_i$, for $i \leq w-1$ and $w^{th}$ row is $\mathbf{1}$. 
\begin{equation}\label{eqn:C_w}
    C_w (\mathbf{x'} - \mathbf{y'}) = \mathbf{e}\,,
\end{equation}
where $i^{th}$ row of $\mathbf{e}$ for $i<w$ is given by some $|e_i| \leq \frac{\epsilon_i + \epsilon_0/2}{1 - \phi_x}$ and the $w^{th}$ entry is 0.
We show that $\mathbf{1}$ is linearly independent of all the $w-1$ other rows. Recall that $\mathbf{s'}_i$ are all $\mu$-normalized, and so $\mathbf{1}^T\mathbf{s'}_i = 0, \forall i$. Therefore, all $w$ rows of $C_w$ are linearly independent, i.e., $C_w$ is full rank and invertible. So, we have
\begin{align}\label{eqn:diff_bound}
    \mathbf{x'} - \mathbf{y'} &= C_w^{-1}\mathbf{e} \implies
    \|\mathbf{x'} - \mathbf{y'}\| = \|C_w^{-1}\mathbf{e}\| \nonumber \\
    &\leq \|C_w^{-1}\| \|\mathbf{e}\| \leq \|C_w^{-1}\|\sqrt{\sum_{i=1}^{w-1} \left(\frac{\epsilon_i + \epsilon_0/2}{1 - \phi_x}\right)^2} =  \delta\,,
\end{align}
which is small for finite $\| C_w^{-1}\|$ and small $\epsilon_i, \forall i \geq 0$.
%Therefore, the only solution to Equation~\ref{eqn:C_w} is $\mathbf{x'} = \mathbf{y'}$.

\end{proof}

\begin{theorem}
    At least $w-1$ linearly independent shapelets are necessary along with the ``flat'' shapelet to satisfy property~\ref{prop:2}.
\end{theorem}
\begin{proof}
We will show that choosing $w-2$ independent shapelets results in at least one $\mathbf{y} \neq \mathbf{x}$, such that $f(\mathbf{x}) = f(\mathbf{y})$ and none of them are flat, i.e., $\phi_x, \phi_y < 1$. Since the vectors are equal across all dimensions, $2\phi_x - 1 = 2\phi_y - 1 \implies \phi_x = \phi_y$. Therefore, for all other dimensions:
\begin{equation}
    (1-\phi_x) \mathbf{s_i'}^T\mathbf{x'} = (1-\phi_y) \mathbf{s_i'}^T\mathbf{y'} \implies \mathbf{s'_i}^T(\mathbf{y'} - \mathbf{x'}) = 0\,.
\end{equation}

Consider the matrix $C_{w-1}$ whose $i^{th}$ row for $i \leq w-2$ is $\mathbf{s'}_i$ and $(w-1)^{th}$ row is $\mathbf{1}$. Since all of its rows are independent, its rank is $w-1$. Using rank-nullity theorem~\cite{divason2013rank}, its nullity is 1.
Therefore, $\exists$ a vector $\mathbf{u} \in \mathbb{R}^w$, with $\|\mathbf{u}\| =  1$ such that,
\begin{equation}\label{eqn:null_rep}
    \mathbf{y'} - \mathbf{x'} = p \mathbf{u}
    \implies \mathbf{y'} = \mathbf{x'} - p \mathbf{u}
\end{equation}
We will prove that there exists a solution to the above other than the trivial solution $p=0$.
First, taking the square of the norm of both sides of Equation~\ref{eqn:null_rep} 
\begin{align}
\|\mathbf{y'}\|^2 &= \|\mathbf{x'} + p \mathbf{u}\|^2 
\implies 1 = \|\mathbf{x'}\|^2 + p^2 \|\mathbf{u}\|^2 + 2p \mathbf{u}^T\mathbf{x'} \nonumber \\
\implies 1 &= 1 + p(p + 2\mathbf{x'}^T\mathbf{u'})\,
\implies p = 0 \mbox{ or } p = -2\mathbf{u}^T\mathbf{x'}\,.
\end{align}
Therefore, for any given $\mathbf{x'}$, if $\mathbf{u}^T\mathbf{x} \neq 0$, there exists $\mathbf{y'} \neq \mathbf{x'}$ which has the same shapelet space representation.  In fact, there are infinitely many such $\mathbf{x'}$ for which this holds. As a demonstration, pick any $\mathbf{x'}$ which is linearly independent with all the $w-1$ rows in $C_{w-1}$ and is not a zero vector. To see that this choice works, note that if $\mathbf{x'}^T\mathbf{u} = 0$, then the nullity of the matrix $C_w'$ formed by appending $x'$ as a row to matrix $C_{w-1}$ is 1 (i.e., $\mathbf{u}$ spans the null space of $C_w$). However, as all $w$ rows of $C_w'$ are linearly independent, its rank is $w$, which violates the rank-nullity theorem.
As a result, there exists a $\mathbf{y'} \neq \mathbf{x'}$ such that $f(\mathbf{x'}) = f(\mathbf{y'})$.
\end{proof}

\section{Experimental Setting}

The chosen shapelets are: (i) `increase': $[1, 2, 3, 4]$, (ii) `surge': $[1, 2, 4, 8]$, (iii) `peak': $[1, 2, 2, 1]$, and (iv) `flat': $[0, 0, 0, 0]$. The matrix $C_w$ constructed as in Theorem~\ref{thm:sufficient}, using the chosen set of shapelets results in $\|C_w^{-1}\| = 13.1$, which is small when multiplied with functions of small $\epsilon_0, \epsilon_1$ as in Equation~\ref{eqn:diff_bound}.
Some other sets of shapelets that satisfy Property~\ref{prop:2} were also tried, and their results were not significantly different. 

The flatness was calculated by setting $m_0 = 0$, and $\beta = -\ln{0.1}/\theta$, where $\theta$ is the median of the maximum ``absolute'' slope of each time-series. Recall that the ``absolute'' slope for a given window is calculated by averaging successive differences over the window. Choosing $\beta$ in such a way ensures that a window of time-series with a median ``absolute'' slope gets a low flatness of 0.1. 

\subsection{Code}
All our code and data used in the paper are publicly available.\footnote{\url{https://github.com/scc-usc/DTW_S_apps}}. The experiments were run on a machine with Intel(R) Core(TM) i8 CPU, 3GHz, 6 cores, and 16GB RAM. 

\section{Additional Results}

\begin{figure}[!t]
    \centering
    \includegraphics[width=\columnwidth, trim = {20 0 335 0}, clip]{val_results.eps}
    \caption{Performance of DTW+S on 64 datasets.}
    \label{fig:val_results_all}
\end{figure}
\begin{figure}[!t]
    \centering
    \includegraphics[width=\columnwidth, trim = {20 0 335 0}, clip]{smooth_results.eps}
    \caption{Results obtained by DTW+S with smoothing. (a) Allowing smoothing reduces errors for many datasets. (b) Improvement on many datasets where DTW+S was worse than DTW.}
    \label{fig:smooth_results_all}
\end{figure}

\subsection{Classification}
Recall that we treat the warping window $\tau$ in DTW+S as a hyperparameter. 
For one set of results, we arbitrarily set $\tau = 10$. For another set of results (shown in the paper)
we use leave-one-out cross-validation to identify $\tau$ as a fraction of the length of the time-series $T$ from a set of values -- $\tau = \{0, 0.01T, 0.02T, \dots, 0.07T\}$. We compare our results with the reported best performance of DTW~\cite{UCRArchive}. The results are presented in Figure~\ref{fig:val_results_all}. In the scatter plot, each point represents errors on a dataset. The blue line represents the line $y=x$, i.e., when DTW+S has the same error as that of DTW. A point lying below (and to the right) of this line indicates that DTW+S error was lower than DTW error.
We observe that DTW+S outperforms DTW on 45.3\% datasets.  The two measures lead to similar errors for many datasets (along the $y=x$ line). With the windows learned through leave-one-out validation the errors improve significantly on many datasets as 
DTW+S outperforms DTW on 57.7\% datasets. ``Corr only'' represents the DTW+S measure obtained by ignoring the ``flat'' shapelet. In this case, DTW+S reduces to the set of Person correlations with the three other shapelets. 
Now, even with the learned warping window, DTW+S (corr) only outperforms DTW in 39.1\% datasets. Additionally, it produces some large errors. Correlation ignores scale completely, and small fluctuations that are not necessarily useful patterns, but noise can cause the measure to consider it similar to some other significant patterns. Thus, the \textit{flatness dimension has a significant contribution} to the performance of DTW+S.  

We also present results (Figure~\ref{fig:smooth_results_all}) highlighting the impact of smoothing vs not smoothing the time-series before classification. We observe that smoothing brings the error down for many datasets. For a few datasets, it makes the results worse. While no-smoothing (smoothing with a window of 1) is also one of the values in the hyper-parameter search, it may not be chosen in validation sometimes because of slightly worse performance in the validation set compared to some other smoothing window.
For easy reference, we also show how smoothing and no smoothing impact errors when compared against DTW.

\end{document}